\documentclass[journal,linesnumbered,vlined,ruled]{IEEEtran}
\ifCLASSINFOpdf
\else
\fi
\usepackage{cite}
\usepackage{amsmath,amssymb,amsfonts,amsthm}
\usepackage{graphicx}
\usepackage{textcomp}
\usepackage{xcolor,colortbl}
\usepackage[linesnumbered,vlined,ruled]{algorithm2e}
\usepackage{balance}
\usepackage{upgreek}
\usepackage{mathtools,breqn}
\usepackage{multirow,arydshln}
\usepackage{bm}
\usepackage{enumerate}
\usepackage{hyperref}
\usepackage{booktabs}
\usepackage{diagbox}
\usepackage{subfig}
\usepackage[figurename=Fig.,font=footnotesize]{caption}

\newcommand{\smallquotes}[1]

\SetAlFnt{\small}

\newtheorem{propo}{Proposition}

\newtheorem{cor}{Corollary}

\theoremstyle{definition}

\newtheorem{rem}{Remark}

\makeatletter
\newcommand{\nosemic}{\renewcommand{\@endalgocfline}{\relax}}
\newcommand{\dosemic}{\renewcommand{\@endalgocfline}{\algocf@endline}}
\let\oldnl\nl
\newcommand{\nonl}{\renewcommand{\nl}{\let\nl\oldnl}}
\def\widebar{\accentset{{\cc@style\underline{\mskip10mu}}}}
\def\wideubar{\underaccent{{\cc@style\underline{\mskip10mu}}}}
\makeatother

\SetCommentSty{mycommfont}
\SetKwProg{Pn}{Function}{:}{\KwRet}
\SetKwProg{Subroutine}{Subroutine}{:}{\KwRet}
\SetKwFunction{TGS}{TGS}
\SetKwFunction{SGS}{SGS+}
\SetKwFunction{RGS}{RGS+}
\SetKwFunction{InitTGS}{InitTGS}
\SetKwFunction{Temper}{Temper}
\SetKwFunction{DistbUdt}{DistbUdt}
\SetKwFunction{IntParm}{IntParm}
\SetKwFunction{FnlParm}{FnlParm}
\SetKwFunction{Unique}{Unique}
\SetKwFunction{supp}{supp}
\SetKwInOut{Input}{Input}
\SetKwInOut{Output}{Output}
\SetKw{continue}{continue}
\SetKw{notand}{or}
\SetKw{and}{and}
\SetKw{then}{then}

\SetKwIF{If}{ElseIf}{Else}{if}{}{else if}{else}{end if}%

\begin{document}

\title{Linear Complexity Gibbs Sampling for\\Generalized Labeled Multi-Bernoulli Filtering}

\author{
        Changbeom~Shim,
        Ba-Tuong~Vo,
        Ba-Ngu~Vo,
        Jonah~Ong, and
        Diluka~Moratuwage
\thanks{This work was supported by the Australian Research Council under Linkage Project LP200301507 and Future Fellowship FT210100506.}
\thanks{The authors are with the School of Electrical Engineering, Computing and Mathematical Sciences, Curtin University, Bentley, WA 6102, Australia (e-mail: changbeom.shim@curtin.edu.au; ba-tuong.vo@curtin.edu.au; ba-ngu.vo@curtin.edu.au; jonahosx25@gmail.com; diluka.moratuwage@curtin.edu.au).}
}

\markboth{PREPRINT: IEEE TRANSACTIONS ON SIGNAL PROCESSING, VOL. 71, 1981-1994, 2023, DOI: 10.1109/TSP.2023.3277220.}{ }

\maketitle

\begin{abstract}
    Generalized Labeled Multi-Bernoulli (GLMB) densities arise in a host of multi-object system applications analogous to Gaussians in single-object filtering. However, computing the GLMB filtering density requires solving NP-hard problems. To alleviate this computational bottleneck, we develop a linear complexity Gibbs sampling framework for GLMB density computation. Specifically, we propose a tempered Gibbs sampler that exploits the structure of the GLMB filtering density to achieve an $\mathcal{O}(T(P+M))$ complexity, where $T$ is the number of iterations of the algorithm, $P$ and $M$ are the number hypothesized objects and measurements. This innovation enables the GLMB filter implementation to be reduced from an $\mathcal{O}(TP^{2}M)$ complexity to $\mathcal{O}(T(P+M+\log T)+PM)$. Moreover, the proposed framework provides the flexibility for trade-offs between tracking performance and computational load. Convergence of the proposed Gibbs sampler is established, and numerical studies are presented to validate the proposed GLMB filter implementation.
\end{abstract}


\IEEEpeerreviewmaketitle

\section{Introduction}\label{s:introduction}

\IEEEPARstart{T}{he} aim of multi-object tracking (MOT) is to estimate the number of objects and their trajectories from noisy sensor data. The challenges are the unknown and time-varying number of objects, accompanied by false alarms, misdetections, and data association uncertainty, culminating in computational bottlenecks for most real world applications. Notwithstanding this, numerous solutions have been developed, with multiple hypothesis tracking \cite{blackman1999designtracking}, joint probabilistic data association \cite{bar2011tracking}, and random finite set (RFS) \cite{mahler2014advances} being the most widely used approaches. The RFS approach, in particular, is gaining substantial interest due to its versatile multi-object state space models \cite{mahler2003multitargetPHD}, and efficient solutions \cite{beard2020solution}. Various tractable RFS multi-object filters have been devised ranging from the probability hypothesis density localization filters \cite{mahler2003multitargetPHD, mahler2007phdCPHD} to the generalized labeled multi-Bernoulli (GLMB) tracking filter~\cite{vo2013labeled}.

The GLMB filter \cite{vo2013labeled} is an analytic solution to the RFS multi-object filter that provides tracks and their \textit{provisional identities/labels} to meet MOT requirements \cite{blackman1999designtracking}. The unique feature of this formulation is the provision for principled trajectory estimation (even with only single-scan filtering), capable of tracking millions of objects, online \cite{beard2020solution}. Moreover, the GLMB family of densities is furnished with elegant mathematical properties, offering a versatile set of tools including, conjugacy \cite{vo2013labeled}, closure under truncation with analytic truncation error \cite{vo2019multiscan}, analytic approximation of general labeled multi-object density with minimal Kullback-Leibler divergence \cite{papi2015generalized}, analytic Cauchy-Schwarz divergence and void probabilities \cite{beard2017void}. These properties enabled MOT with multiple sensors and scans \cite{vo2019multiscan, vo2019multisensor, moratuwage2022multi}, non-standard models \cite{papi2015generalized, li2017multiobject, nguyen2021tracking, ong2020bayesian}, unknown system parameters \cite{trezza2022multi}, multi-object control solutions \cite{beard2017void, van2019online}, as well as distributed implementations \cite{wang2015distributed, li2017robust, li2018computationally}. The GLMB filter is also amenable to parallelization that reduces computation times. For instance, in \cite{beard2020solution}, spatial search was used to decompose the filtering density into independent GLMBs that are processed in parallel, while, in \cite{herrmann2021distributed}, a parallel centralized implementation for multiple sensors was developed. The GLMB filter has also found a host of applications from robotics \cite{deusch2015labeled}, sensor networks \cite{gostar2020cooperative, gostar2020centralized}, cell biology \cite{nguyen2021tracking, hadden2017stem} to audio/video processing \cite{ong2020bayesian, ong2022audio}.

The main computational bottleneck in GLMB filtering is the truncation of the multi-object filtering density \cite{vo2013labeled}. Truncation by discarding terms with small weights minimizes the $L_{1}$-error \cite{vo2019multiscan}, and can be posed as a \textit{ranked assignment problem}, solvable by Murty's algorithm \cite{murty1968letter} with cubic complexity in both the number of measurements and hypothesized objects \cite{miller1997optimizing, pedersen2008algorithm}. A more efficient solution based on Gibbs sampling (GS) was proposed in \cite{vo2017efficient}, which incurs an $\mathcal{O}(TP^{2}M)$ complexity, where $T$, the number of iterates of the algorithm, dominates the number of hypothesized objects and measurements $P$ and $M$. GS is an efficient Markov Chain Monte Carlo (MCMC) technique for sampling from complex probability distributions, popularized by the seminal work of Geman and Geman \cite{geman1984stochastic}, which opened up applications in many disciplines ranging from statistics, engineering to computer science, and is still an active research topic.

Following the strategy of selecting significant GLMB components by random sampling \cite{vo2017efficient}, a number of GLMB truncation techniques have been developed. In \cite{yang2018efficient}, an approximate GLMB filter with linear complexity in the number of hypothesized objects was proposed by neglecting the standard data association requirement of at most one measurement per object. While this technique can be modified to accommodate the data association requirement, the complexity reverts to quadratic in number of hypothesized objects \cite{yang2018efficient}. In \cite{wolf2020deterministic}, a herded Gibbs sampling implementation of the labeled multi-Bernoulli (LMB) filter \cite{reuter2014labeled}--a one-term approximation of the GLMB filter--was developed. However, this implementation is slower than the systematic-scan Gibbs sampling (SGS) GLMB implementation \cite{vo2017efficient}, not to mention that, compared to the (exact) GLMB filter, the LMB filter is more prone to track fragmentation and track switching. In \cite{yu2017algorithms} and \cite{saucan2018distributed}, the cross-entropy method \cite{nguyen2014solving} was applied, respectively, to multi-sensor GLMB filtering and its distributed version, but with higher complexity than the SGS implementation in \cite{vo2019multisensor}.

To alleviate the computational bottleneck in GLMB filtering, this paper proposes a tempered Gibbs sampling (TGS) framework for selecting significant GLMB components, with linear complexity, i.e., $\mathcal{O}(T(P+M))$. Similar to the widely known random-scan Gibbs sampling (RGS), TGS randomly selects a coordinate to update, but provides the mechanism to improve mixing and sample diversity \cite{zanella2019scalable}. However, generic TGS incurs an $\mathcal{O}(TP^{2}M)$ complexity. To this end, we develop an innovative decomposition of the conditionals that enables TGS to be reduced to $\mathcal{O}(T(P+M))$ complexity, with negligible additional memory (compared to generic $\mathcal{O}(TP^{2}M)$ TGS). The samples generated by the proposed TGS algorithm converge to a tempered distribution (not necessarily the same as the original stationary distribution). Furthermore, keeping in mind that TGS is regarded as the combination of importance sampling and MCMC, the importance-weighted samples indeed converge to the original stationary distribution \cite{zanella2019scalable}. The drastic reduction in computational complexity facilitates the development of multi-object control solutions, which require fast computation of the multi-object filtering density for online operations \cite{wang2018multi, van2020multi, panicker2020tracking, wu2020mm}. Moreover, our solution is not only restricted to GLMB truncation, but amenable to a wide range of applications of the ranked assignment problem \cite{burkard2012assignment}.

The proposed TGS framework enables RGS to be implemented with $\mathcal{O}(T(P+M))$ complexity, whereas generic RGS incurs $\mathcal{O}(TPM)$. Further, we propose deterministic-scan Gibbs sampling (DGS), an efficient algorithm using deterministic coordinate selection, as opposed to RGS's completely random coordinate selection. DGS's better mixing and sample diversity compared to RGS are validated by numerical studies. In addition, using DGS, we show how an exact implementation of SGS can be accomplished with $\mathcal{O}(TPM)$ complexity. We also present numerical studies in MOT to discuss the trade-offs between tracking performance and computational load in the proposed TGS framework. Due to the computation of the cost matrix, the resultant GLMB filter implementation incurs an additional $\mathcal{O}(PM)$ complexity. In summary, our main contribution is a TGS framework for GLMB truncation that:
\begin{itemize}
    \item Incurs a linear complexity of $\mathcal{O}(T(P+M))$, with negligible additional memory, by exploiting the structure of the GLMB filtering density; and
    \item  Admits as special cases, RSG, DGS, and SGS implementations, with $\mathcal{O}(T(P+M))$, $\mathcal{O}(TM)$, and $\mathcal{O}(TPM)$ complexities, respectively.
\end{itemize}
We validate the proposed approach via a series of comprehensive numerical experiments, with considerations for computational efficiency, tracking accuracy, and sample diversity.

The rest of this paper is organized as follows. Section \ref{s:preliminaries} provides the necessary background on GLMB filtering and SGS-based GLMB truncation. In Section \ref{s:linear_glmb}, we present linear complexity GS for GLMB filtering density truncation. Numerical studies are presented in Section \ref{s:experiment}, and concluding remarks are given in Section \ref{s:conclusion}.

\section{Preliminaries}\label{s:preliminaries}
This section presents the necessary background for the development of the main result of this article. We first outline the basics of generalized labeled multi-Bernoulli (GLMB) filtering in Subsection \ref{ss:glmb_filtering}, and then summarize the Gibbs sampling (GS) approach to GLMB truncation in Subsection \ref{ss:gibbs_glmb_truncation}. Throughout this paper, single-object states and multi-object states are respectively represented by lower case letters (e.g., $x$) and upper case letters (e.g., $X$). Further, boldfaced symbols denote labeled states or distributions (e.g., $\boldsymbol{x}$, $\boldsymbol{X}$, $\boldsymbol{\pi}$). Frequently used notations are summarized in Table~\ref{tbl:notation}.

\begin{table}[t!]
    \renewcommand{\arraystretch}{1.15}
    \caption{Summary of frequent notations.}
    \vspace*{-0.2cm}
        \footnotesize
        \label{tbl:notation}
        \begin{center}
            \begin{tabular}{|c|l|}
            \hline
                \textbf{\small{Notation}} & \textbf{\small{Description}}\\
            \hline
                $1_S(\cdot)$ & Indicator function for a given set $S$\\
                $|X|$ & Cardinality of a set $X$\\
                $\langle f,g\rangle$ & Inner product, $\int f(x)g(x)dx$ of two functions $f$ and $g$\\
                $\alpha$ & Mixture parameter of proposal distribution\\
                $\beta$ & Tempering parameter of proposal distribution\\
                $\delta_X[Y]$ & Kronecker-$\delta$, $\delta_{X}[Y]=1$ if $X=Y$, $0$ otherwise\\
                $\Delta(\boldsymbol{X})$ & Distinct label indicator $\delta_{|\boldsymbol{X}|}[|\mathcal{L}(\boldsymbol{X})|]$\\
                $\mathcal{F}(S)$ & Class of all finite subsets of a given set $S$\\
                $f(x_{m:n})$ & $f(x_{m}),...,f(x_{n}),f:x_{i}\mapsto f(x_{i})$\\
                $f_{B}(x,\ell)$ & Birth probability density of $x$ given label $\ell$\\
                $\Gamma$ & Space of all (extended) association map $\gamma$\\
                $\gamma$ & (Extended) association map\\
                $\gamma\circ\mathcal{L}(\cdot)$ & Function composition $\gamma(\mathcal{L}(\cdot))$\\
                $g(z|\boldsymbol{x})$ & Single-object likelihood\\
                $g(Z|\boldsymbol{X})$ & Multi-object likelihood\\
                $h^X$ & Multi-object exponential, $\prod_{x\in X}h(x)$, with $h^{\emptyset}=1$\\
                $\mathcal{L}(\boldsymbol{x})$ & Label of a (labeled) state $\boldsymbol{x}$, $\mathcal{L}((x,\ell))=\ell$\\
                $\mathcal{L}(\boldsymbol{X})$ & Labels of a (labeled) set $\boldsymbol{X}$, $\{\mathcal{L}(\boldsymbol{x}):\boldsymbol{x}\in\boldsymbol{X}\}$\\
                $\mathcal{L}(\gamma)$ & Live labels of $\gamma$, $\{\ell\in\mathbb{L}:\gamma(\ell)\geq1\}$\\
                $M$ & Number of measurements (at the next time)\\
                $P$ & Number of hypothesized objects\\
                $P_{B}(\ell)$ & Birth probability of label $\ell$\\
                $P_{D}(\boldsymbol{x})$ & Detection probability of labeled state $\boldsymbol{x}$\\
                $P_{S}(\boldsymbol{x})$ & Survival probability of labeled state $\boldsymbol{x}$\\
                $\boldsymbol{\pi}$ & Multi-object filtering density\\
                $\pi$ & Stationary distribution\\
                $\phi$ & Proposal distribution\\
                $\rho$ & Coordinate probability distribution\\
                $T$ & Number of iterates\\
                $x_{m:n}$ & $x_m, x_{m\text{+}1}, ..., x_n$\\
                $\xi$ & History of association maps\\
                $\Xi$ & Some finite discrete space\\
            \hline
        \end{tabular}
        \vspace*{-0.1cm}
    \end{center}
\end{table}

\subsection{GLMB Filtering}\label{ss:glmb_filtering}
    \subsubsection{Multi-object State}\label{sss:state}
        In GLMB filtering, the system state to be estimated at each time is the set $\boldsymbol{X}$ of labeled states of the underlying objects, called the \textit{multi-object state} \cite{vo2013labeled}. Each labeled single-object state $\boldsymbol{x}\in\boldsymbol{X}$ is an ordered pair $(x,\ell)$ in the product space $\mathbb{\mathbb{X}\times\mathbb{L}}$, where $\mathbb{X}$ is a (finite dimensional) \textit{state space}, and $\mathbb{L}$ is a discrete space called the \textit{label space}. Let $\mathbb{B}_{\tau}$ denote the label space of objects born at time $\tau$, then the space $\mathbb{L}$ of all labels up to time $k$ is given by the disjoint union $\uplus_{\tau\texttt{=}0}^{k}\mathbb{B}_{\tau}$. The state $x$ of an object varies with time, while its label or identity $\ell$ (usually consists of the object's time of birth and an index to distinguish those born at the same time) is time-invariant. The \textit{trajectory} of an object is a sequence of consecutive labeled states with a common label.

        The cardinality $\left|\boldsymbol{X}\right|$ (i.e., number of elements) of the multi-object state $\boldsymbol{X}$ varies with time due to the appearance and disappearance of objects. In addition, a multi-object state $\boldsymbol{X}$ at any time must have distinct labels. More concisely, let $\mathcal{L}\left(\boldsymbol{x}\right)$ denote the label of $\boldsymbol{x}$, and for any finite $\boldsymbol{X}\subset\mathbb{X}\times\mathbb{L}$, define the label set $\mathcal{L}\left(\boldsymbol{X}\right)\triangleq\left\{ \mathcal{L}\left(\boldsymbol{x}\right):\boldsymbol{x}\in\boldsymbol{X}\right\}$ and \textit{distinct label indicator} $\Delta(\boldsymbol{X})\triangleq\delta_{|\boldsymbol{X}|}\big[|\mathcal{L}(\boldsymbol{X})|\big]$. Then for $\boldsymbol{X}$ to be valid multi-object state, we require $\Delta\left(\boldsymbol{X}\right)=1$.
        
        In Bayesian estimation, the single-object state and the measurement are modeled as random variables. Hence, for multi-object estimation, the multi-object state and measurement are modeled as Random Finite Sets (RFSs). In the following we describe the so-called standard multi-object models for the dynamics and observations in a multi-object system.

    \subsubsection{Multi-object Dynamic}\label{sss:dynamic}
        For simplicity, we omit the time subscript ``\textit{k}'', and use the subscript ``+'' for the next time $k+1$ when there is no ambiguity. Each element $\boldsymbol{x}=(x,\ell)$ of the current multi-object state $\boldsymbol{X}$ either survives with probability $P_{S}(\boldsymbol{x})$ and evolves to state $\boldsymbol{x}_{\text{+}}=(x_{\text{+}},\ell_{\text{+}})$ at the next time according to the transition density $f_{S,\text{+}}(x_{\text{+}}|x,\ell)\delta_{\ell}[\ell_{\text{+}}]$, or dies with probability $1-P_{S}(\boldsymbol{x})$ \cite{vo2013labeled}. The term $\delta_{\ell}[\ell_{\text{+}}]$ in the transition density ensures the object retains the same label. In addition to the surviving objects, new objects can be born. New born objects are usually modeled by an LMB RFS, where an object with state $\boldsymbol{x}_{\text{+}}=(x_{\text{+}},\ell_{\text{+}})$ is born at the next time with probability $P_{B,\text{+}}(\ell_{\text{+}})$, and the state density $f_{B,\text{+}}(x_{\text{+}},\ell_{\text{+}})$. The multi-object state $\boldsymbol{X}_{\text{\text{+}}}$ at the next time is the union of surviving objects and new born objects, described by the \textit{multi-object Markov transition density} $\boldsymbol{f}_{\text{+}}\left(\boldsymbol{X}_{\text{+}}|\boldsymbol{X}\right)$ \cite{vo2013labeled}. It is assumed that, conditional on the current multi-object state, objects survive and move independently of each other, and that new born objects and surviving objects are independent \cite{mahler2014advances}.

    \subsubsection{Multi-object Observation}\label{sss:obervation}
        Each element $\boldsymbol{x}\in\boldsymbol{X}$ is either detected with probability $P_{D}(\boldsymbol{x})$ and generates an observation $z$ at the sensor with likelihood $g(z|\boldsymbol{x})$, or misdetected with probability $1-P_{D}(\boldsymbol{x})$. In addition to the detections, the sensor also receives clutter, modeled by a Poisson RFS with intensity function $\kappa$. The multi-object observation $Z$ is the union of detections and clutter. It is assumed that conditional on the multi-object state, objects are detected independently from each other and that clutter and detections are independent.

        The association of objects with sensor measurements (at time $k$) is described by a \textit{positive 1-1} mapping $\gamma\!:\!\mathbb{L}\!\rightarrow\!\{\texttt{\text{-}}1{\textstyle :}|Z|\}$, i.e., a mapping where \textit{no two distinct arguments are mapped to the same positive value} \cite{vo2013labeled}. The (extended) association map\footnote{Originally called extended association maps, herein referred to as association maps for brevity.} $\gamma$ specifies that object (with label) $\ell$ generates measurement $z_{\gamma(\ell)}\in Z$, with $\gamma(\ell)=0$ if it is undetected, and $\gamma(\ell)=-1$ if it does not exist. The positive 1-1 property ensures each measurement comes from at most one object. Let $\mathcal{L}(\gamma)\triangleq\{\ell\in\mathbb{L}:\gamma(\ell)\geq0\}$ denote the set of \textit{live labels}\footnote{Note the distinction from $\mathcal{L}(\boldsymbol{X})$, the labels of a labeled set.} of $\gamma$, and $\Gamma$ denote the space of all association maps, then the \textit{multi-object likelihood function} can be written as \cite{vo2019multisensor}
        \begin{equation}
            g(Z|\boldsymbol{X}) \propto \sum_{\gamma\in\Gamma} \delta_{\mathcal{L}(\gamma)}[\mathcal{L}(\boldsymbol{X})][\psi_{Z}^{(\gamma\circ\mathcal{L}(\cdot))}(\cdot)]^{\boldsymbol{X}},
        \end{equation}
        where $\gamma\circ\mathcal{L}(\cdot)=\gamma(\mathcal{L}(\cdot))$ and
        \begin{equation}
            \psi_{\{z_{1:M}\}}^{(j)}(\boldsymbol{x})=
            \begin{cases}
                \frac{P_{D}(\boldsymbol{x})g(z_{j}|\boldsymbol{x})}{\kappa(z_{j})}, & j>0\\
                1-P_{D}(\boldsymbol{x}), & j=0
            \end{cases}.
        \end{equation}

    \subsubsection{The GLMB Filter}\label{sss:filtering}
        All statistical information about the underlying state is contained in the \textit{filtering density}–probability density of the current state conditioned on all measurements up to (and including) the current time. Given the multi-object filtering density $\boldsymbol{\pi}$ at the current time, the propagation to the next time step is given by \cite{mahler2014advances}
        \begin{equation}
            \boldsymbol{\pi}_{\text{+}}\!\left(\boldsymbol{X}_{\text{+}}|Z_{\text{+}}\right)	\propto g\left(Z_{\text{+}}|\boldsymbol{X}_{\text{+}}\right)\!\int\! \boldsymbol{f}_{\text{+}}\left(\!\boldsymbol{X}_{\text{+}}|\boldsymbol{X}\right)\!\boldsymbol{\pi}\!\left(\boldsymbol{X}\right)\delta\boldsymbol{X}.
        \end{equation}
        
        Under the standard multi-object system model, the multi-object filtering density takes on the GLMB form \cite{vo2013labeled}:
        \begin{equation}\label{eq:GLMB}
            \boldsymbol{\pi}\left(\boldsymbol{X}\right) = \Delta\left(\boldsymbol{X}\right)\sum_{\xi\in\Xi}\left[p^{(\xi)}\right]^{\boldsymbol{X}}\sum_{I\in\mathcal{F}(\mathbb{L})}w^{\left(\xi,I\right)}\delta_{I}[\mathcal{L}\left(\boldsymbol{X}\right)],
        \end{equation}
        where $\Xi$ is some finite discrete space, $p^{\left(\xi\right)}\left(\cdot,\ell\right)$ is a probability density on $\mathbb{X}$, and $w^{\left(\xi,I\right)}$ is a non-negative weight satisfying the condition $\sum_{\left(\xi,I\right)\in\,\Xi\times\mathcal{F}(\mathbb{L})}w^{\left(\xi,I\right)}=1$. To obtain a multi-object state estimate from the GLMB density \eqref{eq:GLMB}, we first find the most probable cardinality $n^{*}$ from the cardinality distribution
        \begin{equation}\label{eq:cardinality_distribution}
            \text{Pr}(|\boldsymbol{X}|=n) \triangleq \sum_{(\xi,I)\in\Xi\times\mathcal{F}(\mathbb{L})}\delta_{n}[|I|]w^{(\xi,I)},
        \end{equation}
        and then the highest-weighted component $(\xi^{*},I^{*})$ with cardinality $|I^{*}|=n^{*}$, see \cite{vo2013labeled}. The state estimate for each object $\ell\in I^{*}$ can be taken as the mean (or mode) of $p^{(\xi^{*})}(\cdot,\ell)$. Alternatively, the entire trajectory of object $\ell\in I^{*}$ can be estimated as described in \cite{vo2019multiscan, van2019online}.

        For compactness, we write a GLMB density in terms of its parameters:
        \begin{equation}\label{eq:GLMB_param}
            \boldsymbol{\pi} \triangleq \left\{ \left(p^{\left(\xi\right)},w^{\left(\xi,I\right)}\right)\right\} {}_{\left(\xi,I\right)\in\,\Xi\times\mathcal{F}(\mathbb{L})}.
        \end{equation}
        Given a current GLMB filtering density of the form \eqref{eq:GLMB_param}, its propagation to the next time is the GLMB \cite{vo2017efficient, vo2019multiscan}
        \begin{equation}\label{eq:GLMB_JPU}
            \boldsymbol{\pi}_{\text{+}}\! = \!\left\{ \!\left(p_{Z_{\text{+}}}^{\left(\xi,\gamma_{\text{+}}\right)},w_{Z_{\text{+}}}^{\left(\xi,\gamma_{\text{+}},I_{\text{+}}\right)}\right)\!\right\} _{\left(\xi,\gamma_{\text{+}},I_{\text{+}}\right)\in\,\Xi\times\Gamma_{\!\text{+}}\mathcal{\times F}\left(\mathbb{L}_{\text{+}}\right)},
        \end{equation}
        where the new parameters are given by
        \allowdisplaybreaks
        \begin{align}
        \phantom{aaaaaaaa}
            &\begin{aligned}
                \mathllap{\!\!p_{Z_{\text{+}}}^{\left(\xi,\gamma_{\text{+}}\right)}\!\left(\cdot,\ell\right)} &= \frac{\bar{p}_{\text{+}}^{\left(\xi,\gamma_{\text{+}}\left(\ell\right)\right)}\left(\cdot,\ell\right)\psi_{Z_{\text{+}}}^{\left(\gamma_{\text{+}}\left(\ell\right)\right)}\left(\cdot,\ell\right)}{\bar{\psi}_{Z_{\text{+}}}^{\left(\xi,\gamma_{\text{+}}\right)}\left(\ell\right)},
            \end{aligned}\\
            &\begin{aligned}\label{eq:weight_recursion}
                \mathllap{\!\!w_{Z_{\text{+}}}^{\left(\xi,\gamma_{\text{+}},I_{\text{+}}\right)}}  &\propto \delta_{\mathcal{L}(\gamma_{\text{+}})}[I_{\text{+}}]\negthinspace\!\negthinspace\sum\limits _{I\in\mathcal{F}(\mathbb{L})}\negthinspace\negthinspace\negthinspace\negthinspace w^{\left(\xi,I\right)}1_{\mathcal{\!F}(I\uplus\mathbb{B}_{\text{+}})\!\!}\left(I_{\text{+}}\right)\!w_{Z_{\text{+}}}^{\left(\xi,I,\gamma_{\text{+}}\right)}\!,
            \end{aligned}\hspace{-0.3cm}\\
            &\begin{aligned}
                \mathllap{\!\!\bar{p}_{\text{+}}^{\left(\xi,j\right)}\!\left(\cdot,\ell\right)} &=
                \begin{cases}
                    \negthinspace\frac{\int\!f_{\text{+}}\!\left(\cdot|x,\ell\right)P_{S}\left(x,\ell\right)p^{\left(\xi\right)}\!\left(x,\ell\right)dx}{\bar{P}_{S}^{\left(\xi\right)}\left(\ell\right)}, & \negthinspace\negthinspace\negthinspace\negthinspace\negthinspace\ell\in I,j\!\geq0\\
                    \negthinspace f_{B,\text{+}}\left(\cdot,\ell\right), & \negthinspace\negthinspace\negthinspace\negthinspace\negthinspace\!\ell\in\mathbb{B}_{\text{+}},j\!\geq0
                \end{cases}\negthinspace,\hspace{-0.3cm}
            \end{aligned}\\
            &\begin{aligned}
                \mathllap{\!\!\bar{P}_{S}^{(\xi)}\left(\ell\right)} &= \left\langle P_{S}\left(\cdot,\ell\right),p^{\left(\xi\right)}\left(\cdot,\ell\right)\right\rangle,
            \end{aligned}\\
            &\begin{aligned}
                \mathllap{\!\!\bar{\psi}_{Z_{\text{+}}}^{\left(\xi,\gamma_{\text{+}}\right)}\left(\ell\right)} &= \left\langle \bar{p}_{\text{+}}^{\left(\xi,\gamma_{\text{+}}\left(\ell\right)\right)}\left(\cdot,\ell\right),\psi_{Z_{\text{+}}}^{\left(\gamma_{\text{+}}\left(\ell\right)\right)}\left(\cdot,\ell\right)\right\rangle,
            \end{aligned}\\
            &\begin{aligned}\label{eq:weight-incr}
                \mathllap{\!\!w_{Z_{\text{+}}}^{\left(\xi,I,\gamma_{\text{+}}\right)}} &= 1_{\Gamma_{\!\text{+}}\!}\left(\gamma_{\text{+}}\right)\prod_{\ell\in I\uplus\mathbb{B}_{\text{+}}}\eta_{Z_{\text{+}_{\!}},\ell}^{(\xi,I)}(\gamma_{\text{+}}(\ell)),
            \end{aligned}\\
            &\begin{aligned}
                \mathllap{\!\!\eta_{Z_{\text{+}},\ell}^{(\xi,I)}(j)} &=
                \!\begin{cases}
                    1-\bar{P}_{S}^{(\xi)\!}(\ell), & \!\ell\in I,\text{ }j\!<0\\
                    \bar{P}_{S\!}^{(\xi)\!}(\ell)\bar{\psi}_{Z_{\text{+}_{\!}}}^{(\xi,j)\!}(\ell_{\!}), & \!\ell\in I,\text{ }j\!\geq0\\
                    1-P_{B\!,\text{+}}(\ell), & \!\ell\in\mathbb{B}_{\text{+}},\text{ }j\!<0\\
                    P_{B\!,\text{+}}(\ell)\bar{\psi}_{Z_{\text{+}}}^{(\xi,j)}(\ell), & \!\ell\in\mathbb{B}_{\text{+}},\text{ }j\!\geq0
                \end{cases}.
            \end{aligned}
        \end{align}

        Note that each component $(\xi,I)$ of the GLMB filtering density at time $k$ generates a (very large) set $\left\{ \left(\xi,I,\gamma_{\text{+}},I_{\text{+}}\right):I_{\text{+}}\in\mathcal{F}(\mathbb{L}_{\text{+}}),\gamma_{\text{+}}\in\Gamma_{\!\text{+}}\right\}$ of children components to the next time. Due to the terms $\delta_{\mathcal{L}(\gamma_{\text{+}})}[I_{\text{+}}]$ and $1_{\mathcal{F}(I\uplus\mathbb{B}_{\text{+}})\!}\left(I_{\text{+}}\right)$ in \eqref{eq:weight_recursion}, we only need to consider components with $I_{\text{+}}\subseteq I\uplus\mathbb{B}_{\text{+}}$ and $\mathcal{L}(\gamma_{\text{+}})=I_{\text{+}}$. While this is a big reduction, in general the total number of GLMB components with non-zero weights still grows super-exponentially with time. Implementing the GLMB filter requires truncating the GLMB filtering density. Truncation by keeping the most highly weighted components minimizes the $L_{1}$ truncation error \cite{vo2019multiscan}.

    \subsection{Gibbs Sampling for GLMB Truncation}\label{ss:gibbs_glmb_truncation}
        Truncating the children of the GLMB component (indexed by $(\xi,I)$) amounts to selecting the $\gamma_{\text{+}}$'s with significant $w_{Z_{\text{+}}}^{\left(\xi,I,\gamma_{\text{+}}\right)}$. For a given component $(\xi,I)$, let us enumerate $I=\{\ell_{1:R}\}$, $\mathbb{B}_{\text{+}\!}=\{\ell_{R\text{+}1:P}\}$, and $Z_{\text{+}}=\{z_{1:M}\}$, and abbreviate
        \begin{equation}\label{eq:eta}
            \eta_{i}(j) \triangleq \eta_{Z_{\text{+}_{\!}},\ell_{i}}^{(\xi,I)}(j),
        \end{equation}
        where $i\in\{1\mathnormal{:}P\}$, and $j\in\{\texttt{-}1\mathnormal{:}M\}$. Let $\pi$ be a (discrete) probability distribution on $\{\texttt{-}1\mathnormal{:}M\}^{P}$ defined by
        \begin{equation}\label{eq:theta_joint_dis}
            \pi(\gamma_{\text{+}}) \propto 1_{{\Gamma_{\!\text{+}}}}(\gamma_{\text{+}})\prod\limits _{i=1\!}^{P}\eta_{i}(\gamma_{\text{+}}(\ell_{i})).
        \end{equation}
        Note that due to the factor $1_{\Gamma\!_{\text{+}}}(\gamma_{\text{+}})$, any sample from \eqref{eq:theta_joint_dis} is a valid association map. Further, it follows from \eqref{eq:weight-incr} that the probability of sampling $\gamma_{\text{+}}$ is $\pi(\gamma_{\text{+}}) \propto w_{Z_{\text{+}}}^{\left(\xi,I,\gamma_{\text{+}}\right)}$. Hence, truncating the contribution from the parent component $(\xi,I)$ can be accomplished by sampling from $\pi$.

        GS is a computationally efficient Markov Chain Monte Carlo (MCMC) technique for sampling from complex probability distributions whose conditionals can be computed/sampled at low cost. In GLMB truncation, we aim to maximize the number of distinct significant samples, rather than focusing on the actual distribution of the samples as per MCMC inference. All distinct samples can be used regardless of their distribution, because each distinct sample constitutes a term in the approximant (the larger the weights, the smaller the approximation error) \cite{vo2017efficient}. Hence, it is not necessary to discard burn-ins and wait for samples from the stationary distribution. 

        Systematic-scan GS (SGS) is the classical approach that samples from the stationary distribution $\pi$ by constructing a Markov chain with transition kernel \cite{geman1984stochastic, casella1992explaining}
        \begin{equation*}
            \pi(\gamma'_{\text{+}}|\gamma_{\text{+}}) 
            = \prod\limits _{i=1}^{P}\pi_{i}(\gamma'_{\text{+}}(\ell_{i})|\gamma'_{\text{+}}(\ell_{1:i\text{-}1}),\gamma_{\text{+}}(\ell_{i\text{+}1:P})),
        \end{equation*}
        where the $i$-th \textit{conditional}, defined on $\{\texttt{-}1\mathnormal{:}M\}$, is given by
        \begin{equation*}
            \pi_{i}(\gamma'_{\text{+}}(\ell_{i})|\gamma'_{\text{+}}(\ell_{1:i\text{-}1}),\gamma_{\text{+}}(\ell_{i\text{+}1:P}))\propto\pi(\gamma'_{\text{+}}(\ell_{1:i}),\gamma_{\text{+}}(\ell_{i\text{+}1:P})).
        \end{equation*}
        This means, for a given $\gamma_{\text{+}}$, the next state $\gamma'_{\text{+}}$ of the chain is generated one component after another, by sampling $\gamma'_{\text{+}}(\ell_{i})$ from $\pi_{i}(\cdot|\gamma'_{\text{+}}(\ell_{1:i\text{-}1}),\gamma_{\text{+}}(\ell_{i\text{+}1:P})), i=1,2, \cdots, P$.

        For GLMB truncation, the conditionals are the categorical distributions given by \cite[Proposition 3]{vo2017efficient}, which is restated in a slightly different form as follows.
        \begin{propo}\label{prop:conditional}
            For each $i\in\{1\mathnormal{:}P\}$, let $\ell_{\bar{i}}$ denote $\ell_{1:i\text{-}1,i\text{+}1:P}$. Then the $i$-th conditional, defined on $\{\texttt{-}1\mathnormal{:}M\}$, is given by
            \begin{equation}\label{eq:i-th_conditional}
                \pi_{i}(\cdot|\gamma_{\text{+}}(\ell_{\bar{i}}))
                = \frac{\widetilde{\pi}_{i}(\cdot|\gamma_{\text{+}}(\ell_{\bar{i}}))}{\left\langle \widetilde{\pi}_{i}(\cdot|\gamma_{\text{+}}(\ell_{\bar{i}})),1\right\rangle },    
            \end{equation}
            where
            \begin{equation}\label{eq:masking}
                \!\!\!\!\widetilde{\pi}_{i}(j|\gamma_{\text{+}}(\ell_{\bar{i}})) \triangleq 
                \begin{cases}
                \eta_{i}(j), & \!j\!<\!1\\
                \eta_{i}(j)(1-1_{\{\gamma_{\text{+}}(\ell_{\bar{i}})\}}(j)), & \!j\!\in\!\{1\mathnormal{:}M\}
                \end{cases}.
            \end{equation}
        \end{propo}

        \begin{rem}\label{rem:mask_1}
            The above result shows that the conditionals are completely characterized by the $P\times(M+2)$ cost matrix in Fig.~\ref{fig:matrix}(a) (which can be pre-computed from the measurement $Z_{\text{+}}$) and the values of $\gamma_{\text{+}}$ on $\ell_{\bar{i}}$, i.e., $\{\gamma_{\text{+}}(\ell_{\bar{i}})\}$. Specifically, the unnormalized $i$-th conditional is simply given by the $i$-th row after entries with (positive) indices contained in $\{\gamma_{\text{+}}(\ell_{\bar{i}})\}$ have been zeroed (or masked) out, as illustrated in Fig.~\ref{fig:matrix}(b). Since evaluating $1_{\{\gamma_{\text{+}}(\ell_{\bar{i}})\}}(j)$ (and hence the mask $1-1_{\{\gamma_{\text{+}}(\ell_{\bar{i}})\}}(j)$) for each $j$ incurs an $\mathcal{O}(P)$ complexity, computing $\pi_{i}(\cdot|\gamma_{\text{+}}(\ell_{\bar{i}}))$ requires an $\mathcal{O}(PM)$ complexity.
        \end{rem}

        \begin{figure}[t!]
            \vspace{-0.1cm}
            \centering
            \includegraphics[width=8.5cm]{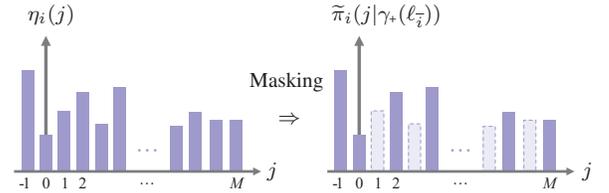}
            \caption{Computing the unnormalized $i$-th conditional. Multiplying the $i$-th row of the matrix in (a) with a masking function (that depends on $\gamma_{\text{+}}$) and normalizing the resulting function in (b) yields the $i$-th conditional.}
            \label{fig:matrix}
        \end{figure}

        \begin{rem}\label{rem:mask_2}
            If $\gamma_{\text{+}}$ is positive 1-1 on $\{\ell_{\bar{i}}\}$ and $\gamma_{\text{+}}(\ell_{i})$ is set to any $j\sim\pi_{i}(\cdot|\gamma_{\text{+}}(\ell_{\bar{i}}))$, then $\gamma_{\text{+}}$ is also positive 1-1, i.e., a valid association map. Multiplication by the mask $1-1_{\{\gamma_{\text{+}}(\ell_{\bar{i}})\}}(j)$ ensures that any $j$ violating the positive 1-1 condition has zero probability of being sampled.
        \end{rem}

        Selecting significant GLMB components can be performed with $\mathcal{O}(TP^{2}M)$ complexity via SGS as shown in \cite{vo2017efficient}. Noting that $P$ (the number of hypothesized objects) is strongly correlated with $M$ (the number of measurements), this complexity translates roughly to a cubic complexity, i.e., $\mathcal{O}(TP^{3})$ or $\mathcal{O}(TM^{3})$. Nonetheless, SGS has been extended to address multi-dimensional ranked assignment problems in multi-scan and multi-sensor GLMB filtering \cite{vo2019multiscan, vo2019multisensor, moratuwage2022multi}.

\section{Linear Complexity GS for GLMB Filtering}\label{s:linear_glmb}
This section presents efficient linear complexity GS for selecting significant components in GLMB filtering. We begin with the widely-known random-scan GS (RGS) in Subsection~\ref{ss:rgs_glmb}. Subsection~\ref{ss:tgs} then presents tempered GS (TGS), a recent generalization that can overcome the drawbacks of RGS, but incurs an $\mathcal{O}(TP^{2}M)$ complexity. In Subsection~\ref{ss:tgs_glmb}, we develop a decomposition of the conditionals (of the stationary distribution) allowing TGS to be implemented with a linear complexity of $\mathcal{O}(T(P+M))$. Salient special cases of the proposed linear complexity TGS, including deterministic-scan GS (DGS), are discussed in Subsection~\ref{ss:special_cases}. For completeness, the linear complexity TGS-based GLMB filter implementation is discussed in Subsection~\ref{ss:filter_imple}.

\subsection{RGS for GLMB Truncation}\label{ss:rgs_glmb}
    Whereas SGS generates the next iterate $\gamma'_{\text{+}}$ by traversing and updating all $P$ coordinates of the current iterate $\gamma_{\text{+}}$, RGS only selects one coordinate at random to update\cite{roberts1998convergence, bryan2016scanorder}. This means the transition kernel $\pi(\gamma'_{\text{+}}|\gamma_{\text{+}})$ is given by $\frac{1}{P}\pi_{i}(\gamma'_{\text{+}}(\ell_{i})|\gamma_{\text{+}}(\ell_{\bar{i}}))$ when $\gamma_{\text{+}}$ and $\gamma'_{\text{+}}$ differ at most in the $i$-th coordinate and $0$ otherwise. A generic RGS implementation would incur an $\mathcal{O}(TPM)$ complexity because computing the conditionals requires $\mathcal{O}(PM)$. Further, RGS is inefficient in the sense that it generates less distinct significant samples than SGS, for the same number of iterates of the chain, leading to poorer GLMB approximations and tracking performance.
    
    The two main factors affecting the efficiency of RGS are \textit{mixing time} and \textit{sample diversity}. Intuitively, mixing time is the number of iterations required for subsequent states to be treated as samples from the stationary distribution $\pi$. Sample diversity refers to the proportion of distinct samples in a given number iterates of the chain. While the actual distribution of the samples is not relevant for GLMB truncation, fast mixing is necessary for efficient generation of distinct significant samples. Furthermore, even if the chain converges to the stationary distribution, sample diversity can be poor due to frequent revisiting of previous states from successive iterations. RGS's notorious slow-mixing\cite{belisle1998slow, roberts1998convergence, zanella2019scalable}, together with observed poor sample diversity means that it could take many iterations for significant GLMB components to be generated. 

\subsection{Tempered Gibbs Sampling}\label{ss:tgs}
    Similar to RGS, TGS also generates the next iterate $\gamma'_{\text{+}}$ by randomly selecting a coordinate to update. However, TGS provides an additional mechanism to improve mixing and sample diversity\cite{zanella2019scalable}. Specifically, a coordinate $i\in\{1\mathnormal{:}P\}$ is chosen according to the distribution
    \begin{equation}\label{eq:selection_probability}
        \rho(i|\gamma_{\text{+}})\propto\frac{\phi_{i}(\gamma_{\text{+}}(\ell_{i})|\gamma_{\text{+}}(\ell_{\bar{i}}))}{\pi_{i}(\gamma_{\text{+}}(\ell_{i})|\gamma_{\text{+}}(\ell_{\bar{i}}))},
    \end{equation}
    where $\phi_{i}(\cdot|\gamma_{\text{+}}(\ell_{\bar{i}}))$ is a bounded proposal, defined on $\{\texttt{-}1\text{:}M\}$, with the same support as $\pi_{i}(\cdot|\gamma(\ell_{\bar{i}}))$. Further, given the selection of the $i$-th coordinate, its state is updated by sampling from the proposal, i.e.,
    \begin{equation}\label{eq:next_gamma_TGS}
        \gamma'_{\text{+}}(\ell_{i})\sim\phi_{i}(\cdot|\gamma_{\text{+}}(\ell_{\bar{i}})),
    \end{equation}
    (note that TGS reduces to RGS in the special case $\phi_{i} = \pi_{i}$). The proposal $\phi_{i}$ controls sample diversity, and determines coordinate selection that can influence mixing\cite{roberts1997updating, diaconis2008gibbs, bryan2016scanorder}. The transition kernel of TGS is given by $\phi(\gamma'_{\text{+}}|\gamma_{\text{+}})\propto\rho(i)\phi_{i}(\gamma'_{\text{+}}(\ell_{i})|\gamma_{\text{+}}(\ell_{\bar{i}}))$ when $\gamma_{\text{+}}$ and $\gamma'_{\text{+}}$ differ at most in the $i$-th coordinate and $0$ otherwise.

    In GLMB truncation, we are not interested in the importance weights since the goal is to generate distinct samples with significant GLMB weights (which are different from the importance weights). While TGS can circumvent the drawbacks of RGS by using fast mixing proposals that yield diverse samples, these may not be significant (in GLMB weights) because the tempered stationary distribution could be very different from $\pi$. One way to generate diverse and significant samples is to use proposals that approximate the conditional $\pi_{i}$, but are more diffuse, which can be achieved with a mixture consisting of the conditional and its tempered version, i.e.,
    \begin{equation}\label{eq:mixture_proposal}
        \phi_{i}(j|\gamma_{\text{+}}(\ell_{\bar{i}})) = \alpha\pi_{i}(j|\gamma_{\text{+}}(\ell_{\bar{i}})) + \frac{(1-\alpha)\pi_{i}^{\beta}(j|\gamma_{\text{+}}(\ell_{\bar{i}}))}{\left\langle \pi_{i}^{\beta}(\cdot|\gamma_{\text{+}}(\ell_{\bar{i}})),1\right\rangle },
    \end{equation}
    where $\alpha,\beta\in(0,1]$, and for any function $f, f^{\beta}(\cdot)\triangleq\left[f(\cdot)\right]^{\beta}$. This popular proposal preserves the modes of the conditional $\pi_{i}(\cdot|\gamma_{\texttt{+}}(\ell_{\bar{i}}))$ to capture significant samples, and at the same time, increases sample diversity via the more diffuse tempered term \cite{gramacy2010importance, zanella2019scalable}. Moreover, the state informed coordinate selection strategy of TGS provides faster mixing \cite{zanella2019scalable, griffin2021search, zhou2022rapid}.

    \begin{algorithm}[t!]
        \caption{Tempered Gibbs Sampling\cite{zanella2019scalable}}
        \label{algo:TGS_general}

        \Input{$\gamma^{}_{\text{+}}$, $\beta$, $[\eta_i(j)]^{P}_{i\texttt{=}1}$}
        \Output{$\gamma'_{\text{+}}$}
        \vspace{0.1cm}\hrule\vspace{0.1cm}
        
        \textsf{\small Compute} $\rho(\cdot|\gamma^{}_{\text{+}})$\;

        \textsf{\small Sample} $n$ from $\rho(\cdot|\gamma^{}_{\text{+}})$\;

        \textsf{\small Sample} $\gamma'_{\text{+}}(\ell_{n})$ from $\phi_{n}(\cdot|\gamma^{}_{\text{+}}(\ell_{\bar{n}}))$\;
    \end{algorithm}

    While the TGS kernel (Algorithm~\ref{algo:TGS_general}) avoids traversing all coordinates, it still incurs an $\mathcal{O}(P^{2}M)$ complexity (and hence generating $T$ iterates incurs $\mathcal{O}(TP^{2}M)$ complexity) because:
    \begin{itemize}
        \item Computing the conditionals $\pi_{i}(\cdot|\gamma_{\text{+}}(\ell_{\bar{i}}))$ for all $i\in\{1\mathnormal{:}P\}$ incurs $\mathcal{O}(P^{2}M)$ complexity since each conditional requires performing the positive 1-1 checks using the set inclusion $1_{\{\gamma_{\text{+}}(\ell_{\bar{i}})\}}(j)$, which incurs an $\mathcal{O}(PM)$ complexity;
        
        \item Computing the proposals $\phi_{i}(\cdot|\gamma_{\text{+}}(\ell_{\bar{i}}))$ for all $i\in\{1\mathnormal{:}P\}$ incurs $\mathcal{O}(PM)$ complexity due to $\beta$-th power operations and normalizations; and
        
        \item Computing the coordinate selection probabilities $\rho(i|\gamma_{\text{+}})$ for all $i\in\{1\mathnormal{:}P\}$ and sampling the coordinate incurs $\mathcal{O}(P)$ complexity, while sampling from $\phi_{n}(\cdot|\gamma_{\text{+}}(\ell_{\bar{n}}))$ incurs $\mathcal{O}(M)$ complexity.
    \end{itemize}
    Nonetheless, it is possible to further exploit the particular structure of the problem through the positive 1-1 constraint to implement this kernel with $\mathcal{O}(P+M)$ complexity. 

\subsection{Linear Complexity TGS}\label{ss:tgs_glmb}
    To reduce the complexity of computing and sampling from the coordinate distribution $\rho(\cdot|\gamma_{\text{+}})$, note that the proposal \eqref{eq:mixture_proposal} can be rewritten in terms of the unnormalized conditional as
    \begin{equation}
        \phi_{i}(j|\gamma_{\text{+}}(\ell_{\bar{i}})) = \frac{\alpha\widetilde{\pi}_{i}(j|\gamma_{\text{+}}(\ell_{\bar{i}}))}{\nu_{i}^{(1)}(\gamma_{\text{+}})} + \frac{(1-\alpha)\widetilde{\pi}_{i}^{\beta}(j|\gamma_{\text{+}}(\ell_{\bar{i}}))}{\nu_{i}^{(\beta)}(\gamma_{\text{+}})},	
    \end{equation}
    where $\nu_{i}^{(\beta)}(\gamma_{\text{+}}) = \left\langle \widetilde{\pi}_{i}^{\beta}(\cdot|\gamma_{\text{+}}(\ell_{\bar{i}})),1\right\rangle$ is the normalizing constant for $\widetilde{\pi}_{i}^{\beta}(\cdot|\gamma_{\text{+}}(\ell_{\bar{i}}))$ (hence $\nu_{i}^{(1)}(\gamma_{\text{+}})$ is the normalizing constant for $\widetilde{\pi}_{i}(\cdot|\gamma_{\text{+}}(\ell_{\bar{i}}))$). Further, the unnormalized conditional $\widetilde{\pi}_{i}(\cdot|\gamma'_{\text{+}}(\ell_{\bar{i}}))$ at the next GS iteration differs from $\widetilde{\pi}_{i}(\cdot|\gamma_{\text{+}}(\ell_{\bar{i}}))$ at no more than two points on its domain $\{\texttt{-}1\mathnormal{:}M\}$. Fig.~\ref{fig:example} provides an illustration.

    \begin{figure}[t!]
        \centering
            \includegraphics[width=8.8cm]{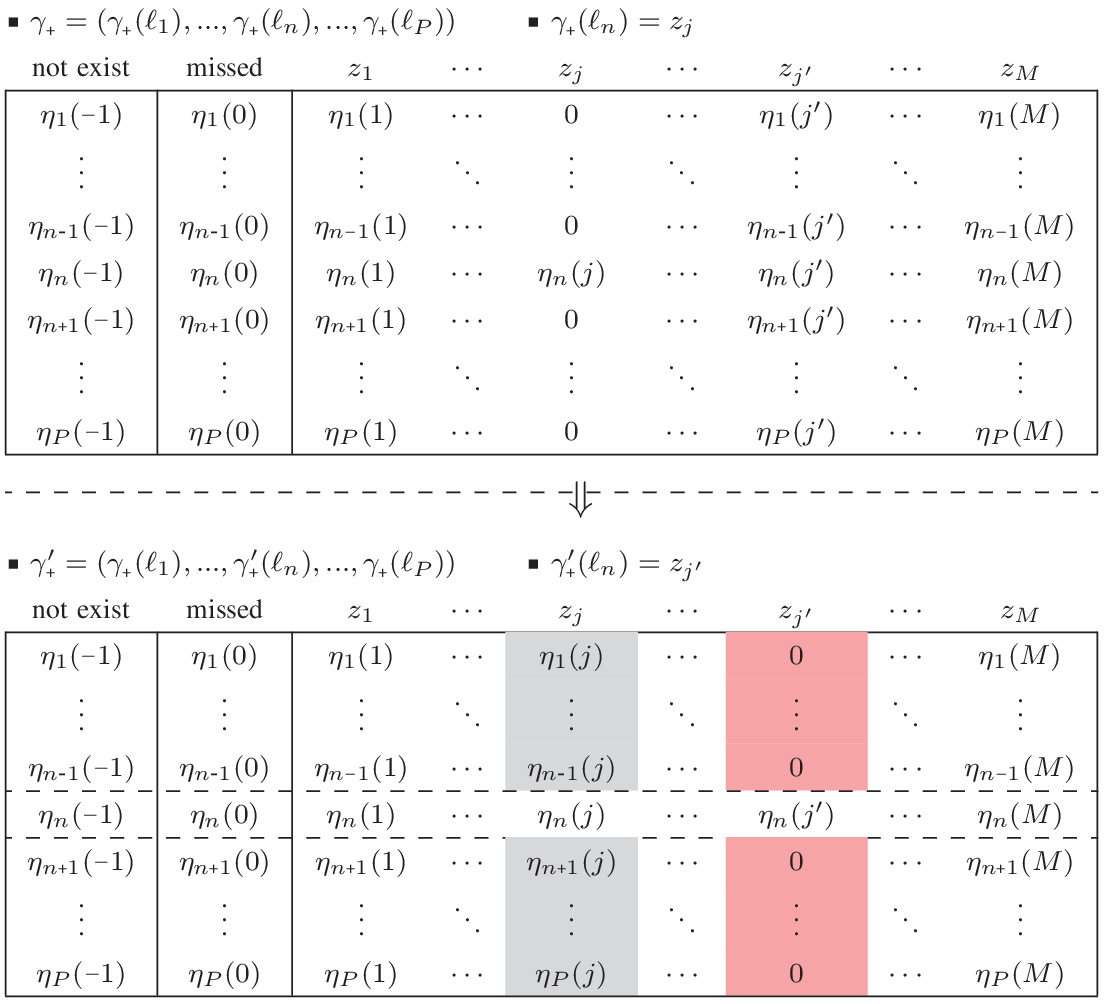}
        \caption{Difference between unnormalized conditionals in successive TGS iterations. If $\gamma'_{\text{+}}$ differs from $\gamma_{\text{+}}$ only at the $n$-th coordinate, i.e., $\gamma_{\text{+}}(\ell_{n})=j>0$ and $\gamma'_{\text{+}}(\ell_{n})=j'>0$, with $j\neq j'$, then for $i\in\{\bar{n}\}$, $\widetilde{\pi}_{i}(\cdot|\gamma'_{\text{+}}(\ell_{\bar{i}}))$ is the same as $\widetilde{\pi}_{i}(\cdot|\gamma_{\text{+}}(\ell_{\bar{i}}))$ on $\{\texttt{-}1\mathnormal{:}M\}$ except at $j$ (\textcolor{gray}{gray}) and $j'$ (\textcolor{red}{red}).}
        \label{fig:example}
    \end{figure}

    \begin{propo}\label{prop:difference}
        Suppose that, $\gamma'_{\text{+}}$ and $\gamma_{\text{+}}$ differ only at the $n$-th coordinate, i.e., $\gamma'_{\text{+}}(\ell_{\bar{n}})=\gamma_{\text{+}}(\ell_{\bar{n}})$ and $\gamma'_{\text{+}}(\ell_{n})\neq\gamma_{\text{+}}(\ell_{n})$. Then, for $i\in\{1\emph{:}P\}$, $j\in\{\texttt{-}1\emph{:}M\}$, and $\beta>0$,
        \begin{align*}
            \hspace{-0.1cm}\widetilde{\pi}_{i}^{\beta}(j|\gamma'_{\text{+}}(\ell_{\bar{i}}))-\widetilde{\pi}_{i}^{\beta}(j|\gamma_{\text{+}}(\ell_{\bar{i}})) &= \\
            & \hspace{-3cm} 1_{\{\bar{n}\}}(i)1_{\{1:M\}}(j)\eta_{i}^{\beta}(j)	\left(\delta_{\gamma_{\text{+}}(\ell_{n})}[j]-\delta_{\gamma'_{\text{+}}(\ell_{n})}[j]\right).
        \end{align*}
    \end{propo}
    \begin{proof}
        See Appendix (Subsection \ref{ss:proof-of-proposition-2}).
    \end{proof}

    \begin{rem}
        Given successive TGS iterates $\gamma_{\text{+}}$ and $\gamma'_{\text{+}}$ that differ at the $n$-th coordinate, Proposition~\ref{prop:difference} (with $\beta=1$) states that their $n$-th conditionals are the same, i.e., $\pi_{n}(\cdot|\gamma_{\text{+}}(\ell_{\bar{n}}))=\pi_{n}(\cdot|\gamma'_{\text{+}}(\ell_{\bar{n}}))$. Moreover, for $i\neq n$, the unnormalized $i$-th conditional $\widetilde{\pi}_{i}(\cdot|\gamma'_{\text{+}}(\ell_{\bar{i}}))$ is the same as $\widetilde{\pi}_{i}(\cdot|\gamma_{\text{+}}(\ell_{\bar{i}}))$, except at:
        \begin{enumerate}[(i)]
            \item $j\!=\gamma_{\text{+}}(\ell_{n})>0$, where $\eta_{i}(j)$ is added to $\widetilde{\pi}_{i}(j|\gamma_{\text{+}}(\ell_{\bar{i}}))$ so that $\widetilde{\pi}_{i}(j|\gamma'_{\text{+}}(\ell_{\bar{i}}))=\eta_{i}(j)$ (because this $j$ is now freed up); and
            
            \item $j\!=\gamma'_{\text{+}}(\ell_{n})>0$, where $-\eta_{i}(j)$ is added to $\widetilde{\pi}_{i}(j|\gamma_{\text{+}}(\ell_{\bar{i}}))$ so that $\widetilde{\pi}_{i}(j|\gamma'_{\text{+}}(\ell_{\bar{i}}))=0$ (because this $j$ is now taken).
        \end{enumerate}
        This also means the normalizing constant for $\widetilde{\pi}_{i}(\cdot|\gamma'_{\text{+}}(\ell_{\bar{i}}))$ can be computed from that of $\widetilde{\pi}_{i}(\cdot|\gamma_{\text{+}}(\ell_{\bar{i}}))$ by simply adding $\eta_{i}(\gamma_{\text{+}}(\ell_{n}))$ when $\gamma_{\text{+}}(\ell_{n})>0$, and subtracting $\eta_{i}(\gamma'_{\text{+}}(\ell_{n}))$ when $\gamma'_{\text{+}}(\ell_{n})>0$. Consequently, the unnormalized conditional $\widetilde{\pi}_{i}(\cdot|\cdot)$ and its normalizing constant can be propagated to the next given TGS iteration with at most two additions.
    \end{rem}

    The above discussion also holds for the tempered unnormalized conditional $\widetilde{\pi}_{i}^{\beta}(\cdot|\cdot)$, and is stated more concisely in the following.
    \begin{cor}\label{cor:decomposition}
        Suppose that the successive TGS iterates $\gamma_{\text{+}}$ and $\gamma'_{\text{+}}$ differ at the $n$-th coordinate. Then, for $\beta>0$,
        \begin{align*}
            \widetilde{\pi}_{n}^{\beta}(\cdot|\gamma'_{\text{+}}(\ell_{\bar{n}}))	&=\widetilde{\pi}_{n}^{\beta}(\cdot|\gamma_{\text{+}}(\ell_{\bar{n}})),\\
            \left\langle \widetilde{\pi}_{n}^{\beta}(\cdot|\gamma'_{\text{+}}(\ell_{\bar{n}})),1\right\rangle &=   \left\langle \widetilde{\pi}_{n}^{\beta}(\cdot|\gamma_{\text{+}}(\ell_{\bar{n}})),1\right\rangle,
        \end{align*}
        and for $i\in\{\bar{n}\}, j\in\{\texttt{-}1\emph{:}M\}$,
        \begin{align*}
            \widetilde{\pi}_{i}^{\beta}(j|\gamma'_{\text{+}}(\ell_{\bar{i}}))	
            &= \begin{cases}
                    \eta_{i}^{\beta}(j), & j=\gamma_{\text{+}}(\ell_{n})>0\\
                    0, & j=\gamma'_{\text{+}}(\ell_{n})>0\\
                    \widetilde{\pi}_{i}^{\beta}(j|\gamma_{\text{+}}(\ell_{\bar{i}})), & \text{otherwise}
                \end{cases}, \\
            \left\langle \widetilde{\pi}_{i}^{\beta}(\cdot|\gamma'_{\text{+}}(\ell_{\bar{i}})),1\right\rangle &= \left\langle \widetilde{\pi}_{i}^{\beta}(\cdot|\gamma{}_{\text{+}}(\ell_{\bar{i}})),1\right\rangle \\
            &\hspace{0.5cm} + \eta_{i}^{\beta}(\gamma_{\text{+}}(\ell_{n}))1_{\{1:M\}}(\gamma_{\text{+}}(\ell_{n})) \\
    		&\hspace{0.5cm} - \eta_{i}^{\beta}(\gamma'_{\text{+}}(\ell_{n}))1_{\{1:M\}}(\gamma'_{\text{+}}(\ell_{n})).
        \end{align*}
    \end{cor}

    \begin{algorithm}[t!]
        \caption{{\small$TGS^+$}}
        \label{algo:TGS_main}
    
        \Input{$\gamma^{(0)}_{\text{+}}$, $T$, $\alpha$, $\beta$, $[\eta_i(j)]^{P}_{i\text{=}1}$, $\rho(\cdot|\gamma^{(0)}_{\text{+}})$, $[\widetilde{\pi}_i(\cdot|\gamma^{(0)}_{\text{+}}\!(\ell_{\bar{i}}))]^{P}_{i\text{=}1}$, $[\nu^{(1)}_i\!(\gamma^{(0)}_{\text{+}})]^{P}_{i\text{=}1}$, $[\nu^{(\beta)}_i\!(\gamma^{(0)}_{\text{+}})]^{P}_{i\text{=}1}$}
        \Output{$[\gamma^{(t)}_{\text{+}}]^{T}_{t\text{=}1}$}
        
        \vspace{0.1cm}\hrule\vspace{0.1cm}
        \SetInd{0.4em}{0.7em}

        $M \leftarrow \textsf{\small size}([\eta_i(j)]^{P}_{i\text{=}1}, \textsf{\footnotesize col})\texttt{-}2$\;
    
        \ForEach{$t = 1:T$}{
            $n \sim \textsf{\small Categorical} ([1\mathnormal{:}P], \rho(\cdot|\gamma^{(t\text{-}1)}_{\text{+}}))$\;\label{line:tgs_cate_sampling}
            
            $\phi_n(\cdot|\gamma^{(t\text{-}1)}_{\text{+}}\!(\ell_{\bar{n}}))  \leftarrow \frac{\alpha\widetilde{\pi}_n(\cdot|\gamma^{(t\text{-}1)}_{\text{+}}\!(\ell_{\bar{n}}))}{\nu^{(1)}_n\!(\gamma^{(t\text{-}1)}_{\text{+}})} +  \frac{(1\text{-}\alpha)\widetilde{\pi}^{\beta}_n\!(\cdot|\gamma^{(t\text{-}1)}_{\text{+}}\!(\ell_{\bar{n}}))}{\nu^{(\beta)}_n\!(\gamma^{(t\text{-}1)}_{\text{+}})}$\;\label{line:tgs_comp_proposal}
            
            $\gamma^{(t)}_{\text{+}}\!(\ell_n) \sim \textsf{\small Categorical} ([\texttt{-}1\mathnormal{:}M], \phi_n(\cdot|\gamma^{(t\text{-}1)}_{\text{+}}\!(\ell_{\bar{n}})))$\;\label{line:tgs_sol_sample}
            
            $\gamma^{(t)}_{\text{+}} \leftarrow [\gamma^{(t\text{-}1)}_{\text{+}}\!(\ell_{1:n\text{-}1}),\gamma^{(t)}_{\text{+}}\!(\ell_n),\gamma^{(t\text{-}1)}_{\text{+}}\!(\ell_{n\text{+}1:P})]$\;
    
            \ForEach{$i = 1:P$}{\label{line:tgs_sel_comp_begin}
                \textsf{\small Compute} $\widetilde{\pi}_i(\cdot|\gamma^{(t)}_{\text{+}}\!(\ell_{\bar{i}})), \nu^{(1)}_i\!(\gamma^{(t)}_{\text{+}}), \nu^{(\beta)}_i\!(\gamma^{(t)}_{\text{+}})$ \phantom{+++++++++++++++++++++++++++} via Corollary~\ref{cor:decomposition}\;
                
                $\pi_i(\gamma_{\text{+}}^{(t)}\!(\ell_i) | \gamma^{(t)}_{\text{+}}\!(\ell_{\bar{i}})) \leftarrow \frac{\widetilde{\pi}_i(\gamma_{\text{+}}^{(t)}\!(\ell_i) | \gamma^{(t)}_{\text{+}}\!(\ell_{\bar{i}}))}{\nu^{(1)}_{i}\!(\gamma^{(t)}_{\text{+}})}$\;
    
                $\phi_i(\gamma_{\text{+}}^{(t)}\!(\ell_i)|\gamma^{(t)}_{\text{+}}\!(\ell_{\bar{i}}))  \leftarrow \alpha\pi_i(\gamma_{\text{+}}^{(t)}\!(\ell_i)|\gamma^{(t)}_{\text{+}}\!(\ell_{\bar{i}}))$ $\phantom{++++++++++++} + \frac{(1\text{-}\alpha) \widetilde{\pi}^{\beta}_i\!(\gamma_{\text{+}}^{(t)}\!(\ell_i) | \gamma^{(t)}_{\text{+}}\!(\ell_{\bar{i}}))}{\nu^{(\beta)}_{i}\!(\gamma^{(t)}_{\text{+}})}$\;
    
                $\widetilde{\rho}(i|\gamma^{(t)}_{\text{+}}) \leftarrow \frac{\phi_i(\gamma_{\text{+}}^{(t)}\!(\ell_i) | \gamma^{(t)}_{\text{+}}\!(\ell_{\bar{i}}))}{\pi_i(\gamma_{\text{+}}^{(t)}\!(\ell_i) | \gamma^{(t)}_{\text{+}}\!(\ell_{\bar{i}}))}$\;\label{line:tgs_sel_comp_end}
            }
            
            $\rho(\cdot|\gamma^{(t)}_{\text{+}}) \leftarrow \frac{\widetilde{\rho}(\cdot|\gamma^{(t)}_{\text{+}})}{\langle \widetilde{\rho}(\cdot|\gamma^{(t)}_{\text{+}}), 1 \rangle}$\;\label{line:tgs_sel_normal}
        }
    \end{algorithm}

    The above result means propagating each tempered conditional and its normalizing constant to the next iterate of the Markov chain can be performed with a constant time complexity. Consequently, $\mathcal{O}(T(P+M))$ complexity TGS can be developed for selecting significant GLMB components. The pseudocode in Algorithm~\ref{algo:TGS_main}, herein referred to as $TGS^{+}$, outlines the steps for generating iterates $\gamma_{\text{+}}^{(1)}, \cdots, \gamma_{\text{+}}^{(T)}$, from initial chain state $\gamma_{\text{+}}^{(0)}$, and initial coordinate distribution $\rho(\cdot|\gamma_{\text{+}}^{(0)})$. Note that:
    \begin{itemize}
        \item Sampling $n$ from the categorical distribution $\rho(\cdot|\gamma_{\text{+}}^{(t\text{-}1)})$, defined on $P$ categories, incurs $\mathcal{O}(P)$ complexity (line~\ref{line:tgs_cate_sampling});
        
        \item  Computing the proposal $\phi_{n}(\cdot|\gamma_{\text{+}}^{(t\text{-}1)}(\ell_{\bar{n}}))$ incurs $\mathcal{O}(M)$ complexity (line~\ref{line:tgs_comp_proposal}) since $\widetilde{\pi}_{i}(\cdot|\gamma_{\text{+}}^{(t\text{-}1)}(\ell_{\bar{i}}))$, $\nu_{i}^{(1)}(\gamma_{\text{+}}^{(t\text{-}1)})$, and $\nu_{i}^{(\beta)}(\gamma_{\text{+}}^{(t\text{-}1)})$ were generated as by-products of computing $\rho(\cdot|\gamma_{\text{+}}^{(t\text{-}1)})$ in the previous iteration;
        
        \item  Sampling $\gamma_{\text{+}}^{(t)}(\ell_{n})$ from $\phi_{n}(\cdot|\gamma_{\text{+}}^{(t\text{-}1)}(\ell_{\bar{n}}))$ incurs $\mathcal{O}(M)$ complexity (line~\ref{line:tgs_sol_sample});
        
        \item Computing $\widetilde{\rho}(\cdot|\gamma_{\text{+}}^{(t)})$ for the $(t+1)$-th iteration is only needed when $\gamma_{\text{+}}^{(t)}(\ell_{n})\neq\gamma_{\text{+}}^{(t\texttt{-}1)}(\ell_{n})$, and requires $\mathcal{O}(P)$ complexity (lines~\ref{line:tgs_sel_comp_begin}-\ref{line:tgs_sel_comp_end}), because for each $i\in\{1\mathnormal{:}P\}$, evaluating \eqref{eq:selection_probability} from the available $\pi_{i}(\gamma_{\text{+}}^{(t)}(\ell_{i})|\gamma_{\text{+}}^{(t)}(\ell_{\bar{i}}))$ and $\phi_{i}(\gamma_{\text{+}}^{(t)}(\ell_{i})|\gamma_{\text{+}}^{(t)}(\ell_{\bar{i}}))$ only requires a constant time complexity (Corollary~\ref{cor:decomposition}); and

        \item   Normalizing $\widetilde{\rho}(\cdot|\gamma_{\text{+}}^{(t)})$ incurs $\mathcal{O}(P)$ complexity (line~\ref{line:tgs_sel_normal}).
    \end{itemize}
    $TGS^{+}$ also requires the initial coordinate distribution $\rho(\cdot|\gamma_{\text{+}}^{(0)})$ as an input, which can be pre-computed via the initialization routine described in Algorithm~\ref{algo:TGS_init}. The recursive construct (lines \ref{line:init_recur_1}-\ref{line:init_recur_2}, \ref{line:init_recur_3}-\ref{line:init_recur_4}) reduces the $\mathcal{O}(P^{2}M)$ complexity \cite{vo2017efficient} to an $\mathcal{O}(PM)$ complexity. Specifically, for each $i\in\{1\mathnormal{:}P\}$, computing $\widetilde{\pi}_{i}(\gamma_{\text{+}}^{(0)}(\ell_{i})|\gamma_{\text{+}}^{(0)}(\ell_{\bar{i}}))$, $\nu_{i}^{(1)}(\gamma_{\text{+}}^{(0)})$, $\nu_{i}^{(\beta)}(\gamma_{\text{+}}^{(0)})$ and $\widetilde{\rho}(i|\gamma_{\text{+}}^{(0)})$ incurs $\mathcal{O}(M)$ complexity.

    \begin{algorithm}[t!]
        \caption{Initialization}
        \label{algo:TGS_init}
        
        \Input{$\gamma^{(0)}_{\text{+}}$, $\alpha$, $\beta$, $[\eta_i(j)]^{P}_{i\text{=}1}$}
        \Output{$[\widetilde{\pi}_i(\cdot|\gamma^{(0)}_{\text{+}}\!(\ell_{\bar{i}}))]^{P}_{i\text{=}1}$, $[\nu^{(1)}_i\!(\gamma^{(0)}_{\text{+}})]^{P}_{i\text{=}1}$, $[\nu^{(\beta)}_i\!(\gamma^{(0)}_{\text{+}})]^{P}_{i\text{=}1}$, $\rho(\cdot|\gamma^{(0)}_{\text{+}})$}
        \vspace{0.1cm}\hrule\vspace{0.1cm}
        \SetInd{0.3em}{0.8em}

        $M \leftarrow \textsf{\small size}([\eta_i(j)]^{P}_{i\text{=}1}, \textsf{\footnotesize col})\texttt{-}2$\;
        
        \ForEach{$j = -1:M$}{\label{line:init_recur_1}
            $\widetilde{\pi}_{P}(j|\gamma^{(0)}_{\text{+}}\!(\ell_{\bar{P}})) \leftarrow 1$\;
        }

        \ForEach{$i = 1:P$}{
            \If{$\gamma^{(0)}_{\text{+}}\!(\ell_i) > 0$}{
                $\widetilde{\pi}_P(\gamma^{(0)}_{\text{+}}\!(\ell_i)|\gamma^{(0)}_{\text{+}}\!(\ell_{\bar{P}})) \leftarrow 0$\;\label{line:init_recur_2}
            }
        }
    
        \ForEach{$i = 1:P$}{
            \ForEach{$j = -1:M$}{
              \uIf{$\widetilde{\pi}_{P}(j|\gamma^{(0)}_{\text{+}}\!(\ell_{\bar{P}})) = 0$ \and $j \neq \gamma^{(0)}_{\text{+}}\!(\ell_{i})$}{\label{line:init_recur_3}
                    $\widetilde{\pi}_{i}(j|\gamma^{(0)}_{\text{+}}\!(\ell_{\bar{i}})) \leftarrow 0$\;
                }
                \Else{
                    $\widetilde{\pi}_{i}(j|\gamma^{(0)}_{\text{+}}\!(\ell_{\bar{i}})) \leftarrow \eta_i(j)$\;\label{line:init_recur_4}
                }
            }
    
            $\pi_i(\gamma^{(0)}_{\text{+}}\!(\ell_i)|\gamma^{(0)}_{\text{+}}\!(\ell_{\bar{i}})) \leftarrow \frac{\widetilde{\pi}_i(\gamma^{(0)}_{\text{+}}\!(\ell_i) | \gamma^{(0)}_{\text{+}}\!(\ell_{\bar{i}}))}{\nu^{(1)}_{i}\!(\gamma^{(0)}_{\text{+}})}$\;
            
            $\phi_i(\gamma^{(0)}_{\text{+}}\!(\ell_i)|\gamma^{(0)}_{\text{+}}\!(\ell_{\bar{i}}))  \leftarrow \alpha\pi_i(\gamma^{(0)}_{\text{+}}\!(\ell_i)|\gamma^{(0)}_{\text{+}}\!(\ell_{\bar{i}}))$ $\phantom{++++++++++++} + \frac{(1\text{-}\alpha) \widetilde{\pi}^{\beta}_i\!(\gamma^{(0)}_{\text{+}}\!(\ell_i)|\gamma^{(0)}_{\text{+}}\!(\ell_{\bar{i}}))}{\nu^{(\beta)}_{i}\!(\gamma^{(0)}_{\text{+}})}$\;
    
            $\widetilde{\rho}(i|\gamma^{(0)}_{\text{+}}) \leftarrow \frac{\phi_i(\gamma^{(0)}_{\text{+}}\!(\ell_i)|\gamma^{(0)}_{\text{+}}\!(\ell_{\bar{i}}))}{\pi_i(\gamma^{(0)}_{\text{+}}\!(\ell_i)|\gamma^{(0)}_{\text{+}}\!(\ell_{\bar{i}}))}$\;
        }
        
        $\rho(\cdot|\gamma^{(0)}_{\text{+}}) \leftarrow \frac{\widetilde{\rho}(\cdot|\gamma^{(0)}_{\text{+}})}{\langle \widetilde{\rho}(\cdot|\gamma^{(0)}_{\text{+}}), 1 \rangle}$\;
    \end{algorithm}

    \begin{propo}\label{prop:convergence}
        Starting from any positive 1-1 initialization, the Markov Chain $\{\gamma_{\text{+}}^{(t)}\}_{t\texttt{=}1}^{\infty}$ generated by the TGS kernel (Algorithm \ref{algo:TGS_main}) is ergodic, and converges to a stationary distribution (not necessarily $\pi$). Additionally, taking into account the importance weight $w^{(t)}\propto P/\left\langle \rho(\cdot|\gamma_{\text{+}}^{(t)}),1\right\rangle$, the weighted samples converge to the stationary distribution $\pi$ in the sense that for any bounded function $h:\Gamma\!_{\text{+}}\rightarrow\mathbb{R}$
        \begin{equation}\label{eq:importance_tempering-TGS}
            \lim_{T\rightarrow\infty}\sum_{t\texttt{=}1}^{T}w^{(t)}h(\gamma_{\text{+}}^{(t)}) = \sum_{\gamma_{\text{+}}\in\Gamma\!_{\text{+}}}\pi(\gamma_{\text{+}})h(\gamma_{\text{+}}),
        \end{equation}
        where the coordinate probability distribution $\rho(\cdot|\gamma_{\text{+}}^{(t)})$ is given by \eqref{eq:selection_probability}. Further, the variance of the importance weights is stable, i.e., does not grow with the number of coordinates.
    \end{propo}
    \begin{proof}
        See Appendix (Subsection \ref{ss:proof-of-proposition-3}).
    \end{proof}
    \begin{rem}
        In TGS, the variance of the weights does not grow with the number of coordinates, and hence the algorithm scales gracefully with $P$. Combining tempering with importance sampling is a strategy used to improve slow mixing in generic MCMC \cite{zanella2019scalable}. However, the variance of the weights grows exponentially with the number of coordinates \cite{tokdar2010importance, owen2013monte}, which means that the performance of a generic MCMC method deteriorates with large values of $P$. Unlike generic tempering in MCMC, TGS only tempers the conditional of the selected coordinate, but still inherits the benefit of improved mixing~\cite{zanella2019scalable}.
    \end{rem}

\subsection{Salient Special Cases}\label{ss:special_cases}
    \subsubsection{Random-scan GS}\label{ss:rgs+}
        \begin{algorithm}[t!]
            \caption{{\small$RGS^+$} }
            \label{algo:RGS}
        
            \Input{ $\gamma^{(0)}_{\text{+}}$, $T$, $[\eta_i(j)]^{P}_{i\text{=}1}$, $[\widetilde{\pi}_i(\cdot|\gamma^{(0)}_{\text{+}}\!(\ell_{\bar{i}}))]^{P}_{i\text{=}1}$ }
            \Output{$[\gamma^{(t)}_{\text{+}}]^{T}_{t\text{=}1}$}
            \vspace{0.1cm}\hrule\vspace{0.1cm}
            
            \SetInd{0.3em}{0.8em}
            $M \leftarrow \textsf{\small size}([\eta_i(j)]^{P}_{i\text{=}1}, \textsf{\footnotesize col})\texttt{-}2$\;
        
            \ForEach{$t = 1:T$}{
                $n \sim \mathcal{U}(1,P)$\;\label{line:rgs_cate_sampling}
                
                $\gamma^{(t)}_{\text{+}}\!(\ell_n) \sim \textsf{\small Categorical} ([\texttt{-}1\mathnormal{:}M], \widetilde{\pi}_n(\cdot|\gamma^{(t\text{-}1)}_{\text{+}}\!(\ell_{\bar{n}})))$\;\label{line:rgs_sol_sample}
                
                $\gamma^{(t)}_{\text{+}} \leftarrow [\gamma^{(t\text{-}1)}_{\text{+}}\!(\ell_{1:n\text{-}1}),\gamma^{(t)}_{\text{+}}\!(\ell_n),\gamma^{(t\text{-}1)}_{\text{+}}\!(\ell_{n\text{+}1:P})]$\;
        
                \ForEach{$i = 1:P$}{\label{line:rgs_decompo_begin}
                    \textsf{\small Compute} $\widetilde{\pi}_i(\cdot|\gamma^{(t)}_{\text{+}}\!(\ell_{\bar{i}}))$ via Corollary~\ref{cor:decomposition}\;\label{line:rgs_decompo_end}
                }
            }
        \end{algorithm}
        Under the proposed TGS framework, setting $\beta=1$ in the proposal \eqref{eq:mixture_proposal} translates to uniformly random coordinate selection. This special case, described by Algorithm \ref{algo:RGS}, is indeed RGS, herein referred to as $RGS^{+}$ to distinguish it from generic $\mathcal{O}(TPM)$ complexity RGS. Since $\rho(\cdot|\gamma_{\text{+}})$ is a uniform distribution, it requires no computation. However, the masking steps for $\widetilde{\pi}_{i}(\cdot|\gamma_{\text{+}}(\ell_{\bar{i}}))$ are still needed. Nonetheless the complexity of $RGS^{+}$ reduces to $\mathcal{O}(T(P+M))$ because:
        \begin{itemize}
            \item Sampling $n$ uniformly from $\{1\mathnormal{:}P\}$ incurs $\mathcal{O}(P)$ complexity (line \ref{line:rgs_cate_sampling});
            \item Sampling $\gamma_{\text{+}}^{(t)}(\ell_{n})$ from $\widetilde{\pi}_{n}(\cdot|\gamma_{\text{+}}^{(t\text{-}1)}\!(\ell_{\bar{n}}))$ incurs $\mathcal{O}(M)$ complexity (line~\ref{line:rgs_sol_sample}); and
            \item  Computing $[\widetilde{\pi}_{i}(\cdot|\gamma_{\text{+}}(\ell_{\bar{i}}))]_{i\text{=}1}^{P}$ (based on Corollary~\ref{cor:decomposition}) takes $\mathcal{O}(P)$ complexity (lines~\ref{line:rgs_decompo_begin}-\ref{line:rgs_decompo_end}).
        \end{itemize}

    \subsubsection{Deterministic-scan GS}\label{ss:dgs}
        The other extreme of complete randomness, namely deterministic coordinate selection can also be accommodated by $TGS^{+}$. The simplest deterministic scan is to traverse the coordinates according to the periodic sequence $1, 2, \cdots, P$, i.e., the selected coordinate at $t$-th iteration is given by $c(t)=1+\text{mod}(t\texttt{-}1,P)$. To realize this in the TGS framework, a proposal of the form
        \begin{equation*}
            \varphi_{i}(j|\gamma_{\text{+}}^{(t)}(\ell_{\bar{i}})) = \delta_{c(t)}[i]\phi_{i}(j|\gamma_{\text{+}}^{(t)}(\ell_{\bar{i}})) + \varepsilon(1-\delta_{c(t)}[i]),
        \end{equation*}
        where $\varepsilon$ is a very small positive number. Note that the corresponding coordinate distribution is given by
        \begin{equation*}
            \rho(i|\gamma_{\text{+}}) \propto \frac{\delta_{c(t)}[i]\phi_{i}(\gamma_{\text{+}}^{(t)}(\ell_{i})|\gamma_{\text{+}}^{(t)}(\ell_{\bar{i}}))}{\pi_{i}(\gamma_{\text{+}}^{(t)}(\ell_{i})|\gamma_{\text{+}}^{(t)}(\ell_{\bar{i}}))}+\frac{\varepsilon(1-\delta_{c(t)}[i])}{\pi_{i}(\gamma_{\text{+}}^{(t)}(\ell_{i})|\gamma_{\text{+}}^{(t)}(\ell_{\bar{i}}))}.
        \end{equation*}
        Hence, in the limiting case as $\varepsilon$ tends to zero, only coordinate $c(t)$ can be selected at the $t$-th iterate. 

        In principle, this special case has $\mathcal{O}(T(P+M))$ complexity. However, since the coordinate is effectively selected according to the prescribed sequence $c(t)$, the sampling step is not needed. Moreover, using the recursive construct in the initialization (Algorithm~\ref{algo:TGS_init}), it is possible to implement this so-called $DGS^{+}$ special case as described in Algorithm~\ref{algo:DGS}, which only incurs an $\mathcal{O}(TM)$ complexity. A similar approach based on the reverse scan order, i.e., $P, P-1, \cdots, 1$, is possible by setting $c(t)=P-\text{mod}(t\texttt{-}1,P)$ in Algorithm~\ref{algo:DGS} (line~\ref{line:dgs_c(t)}).

        \begin{algorithm}[t!]
            \caption{{\small$DGS^{+}$} }
            \label{algo:DGS}
            
            \Input{$\gamma^{(0)}_{\text{+}}$, $T$, $\alpha$, $\beta$, $[\eta_i(j)]^{P}_{i\text{=}1}$, $[\widetilde{\pi}_i(\cdot|\gamma^{(0)}_{\text{+}}\!(\ell_{\bar{i}}))]^{P}_{i\text{=}1}$, $[\nu^{(1)}_i\!(\gamma^{(0)}_{\text{+}})]^{P}_{i\text{=}1}$, $[\nu^{(\beta)}_i\!(\gamma^{(0)}_{\text{+}})]^{P}_{i\text{=}1}$}
            \Output{$[\gamma^{(t)}_{\text{+}}]^{T}_{t\text{=}1}$}
            \vspace{0.1cm}\hrule\vspace{0.1cm}
            \SetInd{0.3em}{0.8em}
            
            $M \leftarrow \textsf{\small size}([\eta_i(j)]^{P}_{i\text{=}1}, \textsf{\footnotesize col})\texttt{-}2$\;
    
            \ForEach{$t = 1:T$}{
                $n \leftarrow c(t)$; \,  $m \leftarrow c(t-1)$\;\label{line:dgs_c(t)}
                
                \If{$\gamma^{(t\text{-}1)}_{\text{+}}\!(\ell_m) > 0$}{\label{line:dgs_recursive2_begin}
                    $\widetilde{\pi}_{m}(\gamma^{(t\text{-}1)}_{\text{+}}\!(\ell_m)|\gamma^{(t\text{-}1)}_{\text{+}}\!(\ell_{\bar{m}})) \leftarrow 0$\;
                }\label{line:dgs_recursive2_end}
                
                \ForEach{$j = -1:M$}{
                
                    \uIf{$\widetilde{\pi}_{m}(j|\gamma^{(t\text{-}1)}_{\text{+}}\!(\ell_{\bar{m}})) = 0$ \and $j \neq \gamma^{(t\text{-}1)}_{\text{+}}\!(\ell_{n})$}{\label{line:dgs_recursive1_begin}
                        $\widetilde{\pi}_{n}(j|\gamma^{(t\text{-}1)}_{\text{+}}\!(\ell_{\bar{n}})) \leftarrow 0$\;
                    }
                    \Else{
                        $\widetilde{\pi}_{n}(j|\gamma^{(t\text{-}1)}_{\text{+}}\!(\ell_{\bar{n}})) \leftarrow \eta_n(j)$\;
                    }\label{line:dgs_recursive1_end}
                }
                
                $\phi_n(\cdot|\gamma^{(t\text{-}1)}_{\text{+}}\!(\ell_{\bar{n}}))  \leftarrow \frac{\alpha\widetilde{\pi}_n(\cdot|\gamma^{(t\text{-}1)}_{\text{+}}\!(\ell_{\bar{n}}))}{\nu^{(1)}_n\!(\gamma^{(t\text{-}1)}_{\text{+}})} +  \frac{(1\text{-}\alpha)\widetilde{\pi}^{\beta}_n\!(\cdot|\gamma^{(t\text{-}1)}_{\text{+}}\!(\ell_{\bar{n}}))}{\nu^{(\beta)}_n\!(\gamma^{(t\text{-}1)}_{\text{+}})}$\;
        
                $\gamma^{(t)}_{\text{+}}\!(\ell_n) \sim \textsf{\small Categorical}([\texttt{-}1$:$M], \phi_n(\cdot|\gamma^{(t\text{-}1)}_{\text{+}}\!(\ell_{\bar{n}})))$\;   \label{line:dgs_categorical}
                
                $\gamma^{(t)}_{\text{+}} \leftarrow [ \gamma^{(t\text{-}1)}_{\text{+}}\!(\ell_{1:n\text{-}1}), \gamma^{(t)}_{\text{+}}\!(\ell_{n}), \gamma^{(t\text{-}1)}_{\text{+}}\!(\ell_{n\text{+}1:P})]$\;
            }
        \end{algorithm}

        Note that in the literature, the terms ``Systematic-scan GS'', ``Sequential-scan GS'', and ``Deterministic-scan GS'' all refer to ``SGS''. The proposed $DGS^{+}$ is clearly not the same as SGS, but is closely related to it. Specifically, the sequence of every $P$-th iterate of $DGS^{+}$, with $\alpha=1$ or $\beta=1$, is indeed a sequence of SGS iterates. However, this approach to SGS (summarized in Algorithm~\ref{algo:SGS}), referred to as $SGS^{+}$, has to traverse $P$ coordinates to generate one iterate. Hence, for $T$ iterations of $DGS^{+}$, $SGS^{+}$ effectively has $T/P$ iterations. This also means that the complexity per iteration increases by $P$, and consequently $SGS^{+}$ incurs an effective $\mathcal{O}(TPM)$ complexity, a drastic reduction from $\mathcal{O}(TP^{2}M)$ as in \cite{vo2017efficient}.

        \begin{algorithm}[t!]
            \caption{{\small$SGS^+$}}
            \label{algo:SGS}
            
            \Input{ $\gamma^{(0)}_{\text{+}}$, $T$, $[\eta_i(j)]^{P}_{i\text{=}1}$,  $[\widetilde{\pi}_i(\cdot|\gamma^{(0)}_{\text{+}}\!(\ell_{\bar{i}}))]^{P}_{i\text{=}1}$ }
            \Output{$[\gamma^{(t)}_{\text{+}}]^{T}_{t\text{=}1}$}
            \vspace{0.1cm}\hrule\vspace{0.1cm}
            \SetInd{0.3em}{0.8em}
    
            $M \leftarrow \textsf{\small size}([\eta_i(j)]^{P}_{i\text{=}1}, \textsf{\footnotesize col})\texttt{-}2$\;
            
            \ForEach{$t = 1:T$}{
    
                $\gamma^{(t)}_{\text{+}}\!(\ell_P) \leftarrow \gamma^{(t\text{-}1)}_{\text{+}}\!(\ell_P)$\;
            
                \ForEach{$n = 1:P$}{
                
                    $m \leftarrow 1+\textsf{\small mod}(P\texttt{-}2\texttt{+}n,P)$\;
    
                    \If{$\gamma^{(t)}_{\text{+}}\!(\ell_m) > 0$}{\label{line:sgs_recursive2_begin}
                        $\widetilde{\pi}_{m}(\gamma^{(t)}_{\text{+}}\!(\ell_m)|\gamma^{(t)}_{\text{+}}\!(\ell_{1\colon n\text{-}1}),\gamma^{(t\text{-}1)}_{\text{+}}\!(\ell_{n\text{+}1\colon P})) \leftarrow 0$\;
                    }\label{line:sgs_recursive2_end}
                    
                    \ForEach{$j = -1:M$}{
                        \uIf{$\widetilde{\pi}_m(j|\gamma^{(t)}_{\text{+}}\!(\ell_{1\colon m\text{-}1}),\gamma^{(t\text{-}1)}_{\text{+}}\!(\ell_{m\text{+}1\colon P})) = 0$}{\label{line:sgs_recursive1_begin}
                            \hspace{4.0cm}\and $j \neq \gamma^{(t\text{-}1)}_{\text{+}}\!(\ell_{n})$
                            
                            $\widetilde{\pi}_{n}(j|\gamma^{(t)}_{\text{+}}\!(\ell_{1\colon n\text{-}1}),\gamma^{(t\text{-}1)}_{\text{+}}\!(\ell_{n\text{+}1\colon P})) \leftarrow 0$\;
                        }
                        \Else{
                            $\widetilde{\pi}_{n}(j|\gamma^{(t)}_{\text{+}}\!(\ell_{1\colon n\text{-}1}),\gamma^{(t\text{-}1)}_{\text{+}}\!(\ell_{n\text{+}1\colon P})) \leftarrow \eta_n(j)$\;
                        }\label{line:sgs_recursive1_end}
                    }
                    
                    $\gamma^{(t)}_{\text{+}}\!(\ell_n) \sim \textsf{\small Categorical}([\texttt{-}1$:$M],$   \label{line:sgs_categorical}
                    \hspace{2.8cm}$\widetilde{\pi}_n(\cdot|\gamma^{(t)}_{\text{+}}\!(\ell_{1\colon n\text{-}1}),\gamma^{(t\text{-}1)}_{\text{+}}\!(\ell_{n\text{+}1\colon P}))$\;
        
                }
                $\gamma^{(t)}_{\text{+}} \leftarrow [\gamma^{(t)}_{\text{+}}\!(\ell_{1:P})]$\;
            }
        \end{algorithm}

        \begin{rem}
            Apart from the forward-scan coordinate selection discussed above, other deterministic scan orders are possible. For efficiency and sample diversity, it is important that the periodic coordinate sequence $c(t)$ scans all $P$ coordinates in one period.
        \end{rem}

\subsection{TGS-GLMB Filter Implementation}\label{ss:filter_imple}
    The TGS-based implementation of the GLMB filter is the same as the SGS-based implementation in \cite[Algorithm 2]{vo2017efficient}, with two key changes:
    \begin{itemize}
        \item SGS-based truncation is replaced by the proposed $\mathcal{O}(T(P+M))$ TGS algorithm (Algorithm \ref{algo:TGS_main}); and
        \item There is an additional $\mathcal{O}(PM)$ complexity step for computing the initial coordinate distribution (Algorithm \ref{algo:TGS_init}). 
    \end{itemize}    
    Thus, the resulting GLMB filter implementation incurs an overall complexity of $\mathcal{O}(T(P+M)+PM+T\log T)$. The additional factor of $T\log T$ arises from the need to remove duplicates from the TGS output. Noting that in practice, the number $T$ of iterates is large, typically $T(P+M)>>PM$, nonetheless $\log T$ is small, and $P+M>>\log T$. This is a substantial reduction from the $\mathcal{O}(TP^{2}M)$ complexity of the SGS-based implementation.

\section{Experimental Results}\label{s:experiment}
    \begin{figure*}[t!]
        \centering
            \includegraphics[width=18.1cm]{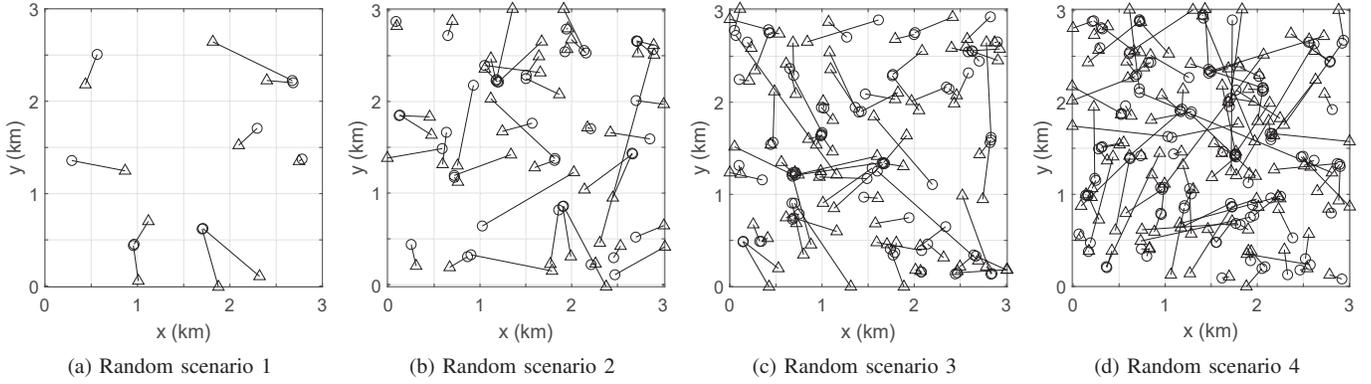}
        \caption{Examples of randomly generated ground truths ($\bigcirc/\bigtriangleup$: start/end).}
        \label{test:truth}
    \end{figure*}

This section presents numerical studies for GLMB filtering with truncation via the proposed TGS based schemes and the new recursive implementation of SGS:
\begin{enumerate}[1)]
    \item $TGS^{+}$: $\mathcal{O}(T(P+M))$ TGS (Algorithm~\ref{algo:TGS_main});
    \item $RGS^{+}$: $\mathcal{O}(T(P+M))$ RGS (Algorithm~\ref{algo:RGS});
    \item $\overrightarrow{DGS}$$^{+}$: $\mathcal{O}(TM)$ forward-scan DGS (Algorithm~\ref{algo:DGS});
    \item $\overleftarrow{DGS}$$^{+}$: $\mathcal{O}(TM)$ backward-scan DGS (Algorithm~\ref{algo:DGS});
    \item $SGS^{+}$: $\mathcal{O}(TPM)$ SGS (Algorithm~\ref{algo:SGS});
    \item $RGS$: $\mathcal{O}(TPM)$ RGS (Subsection~\ref{ss:rgs_glmb}); and
    \item $SGS$: $\mathcal{O}(TP^{2}M)$ SGS (Algorithm~2a in \cite{vo2017efficient}).
\end{enumerate}
Note that, although all methods use the same initialization (Algorithm~\ref{algo:TGS_init}), only $TGS^{+}$ requires $\rho(\cdot|\gamma_{\text{+}})$, whereas $RGS^{+}/SGS^{+}$ do not require the tempering related steps. No parallel implementations are used in our experiments.

A common linear Gaussian setup on a square $3~km~\times~3~km$ surveillance region over 100 time steps is used throughout this section. The number of objects is time varying due to various births and deaths. The single-object state is a 4D vector $[x, \dot{x}, y, \dot{y}]^{\textsf{T}}$ of 2D position and velocity which follows a constant velocity model with sampling period $1$ $s$ and process noise standard deviation $\sigma_{p}=5$~$ms^{-2}$ on each axis. Existing objects have a constant survival probability of $P_{S}=0.99$. New objects appear according to an LMB birth density with $N_{B}=50$ components. Each LMB component has a fixed birth probability $P_{B,\texttt{+}}(\ell_{i})=0.01$ and Gaussian birth density $f_{B,\texttt{+}}(x,\ell_{i})=\mathcal{N}(\cdot;m_{B}^{(i)},R_{B})$ for $i\in\{1,...,N_{B}\}$, where the means $m_{B}^{(i)}=[x^{(i)}, 0, y^{(i)}, 0]^{\textsf{T}}$ are uniformly spaced in the surveillance region, and $R_{B}=\mathrm{diag}([10,10,10,10]^{\textsf{T}})^{2}$. This implies a mean birth rate of $\lambda_{B}=0.5$ per scan, combined with the constant death or survival probability, results in a total of approximately $N_{X}=50$ trajectories, or a peak of approximately $33$ objects simultaneously. Observations are noisy 2D positions with noise standard deviation $\sigma_{m}=10$~$m$ on each axis. All objects have a constant detection probability of $P_{D}=0.86$. Clutter follows a Poisson model with a mean rate of $\lambda_{c}=90$ returns per scan and a uniform density on the observation space. Each of the GS implementations are run for $T=5000$ iterations, and the $TGS^{+}/\overrightarrow{DGS}$$^{+}/\overleftarrow{DGS}$$^{+}$ use default mixture and tempering parameters of $\alpha=\beta=0.5$.

The effects of individually varying key parameters affecting $T$, $P$, and $M$ are also examined. In addition to numerical studies with the default parameters described above, further experiments are also carried out according to Table~\ref{tbl:parameter} with varying number of GS iterations, total number of trajectories (or peak number of objects), plus detection and clutter rates. The effects of different mixture rates and tempering levels for $TGS^{+}$ are also studied.

\begin{table}[h!]
\renewcommand{\arraystretch}{1.15}
\caption{Parameter settings.}
\vspace*{-0.2cm}
    \footnotesize
    \label{tbl:parameter}
    \begin{center}
        \begin{tabular}{|c|c|}
            \hline
                \multicolumn{1}{|c|}{\textbf{\small{Parameter}}}   &    \multicolumn{1}{c|}{\textbf{\small{Setting (default)}}}   \\
            \hline \hline  
                Number of Iterations $T$ &  $1000$ - $10000$ $(5000)$ \\
            \hline
                Total Trajectories $N_{X}$ &  $10$ - $100$ $(50)$ \\
            \hline
                Detection Rate $P_{D}$ &  $0.78$ - $0.96$ $(0.86)$ \\
            \hline
                Clutter Rate $\lambda_{c}$ &  $50$ - $140$ $(90)$ \\
            \hline
                Mixture Rate $\alpha$  &   $0.1$ - $0.9$ $(0.5)$ \\
            \hline
                Tempering Level $\beta$ &  $0.1$ - $0.9$ $(0.5)$ \\
            \hline
        \end{tabular}
    \end{center}
\end{table}

Evaluations are undertaken for computation time, tracking accuracy, and (association) sample diversity via (i) measured wall-clock times, (ii) OSPA \cite{schuhmacher2008consistent} and OSPA$^{(2)}$ \cite{beard2020solution} metrics, and (iii) the number of unique solutions per scan during GLMB filtering. Results on the tracking accuracy and sample diversity of $RGS/SGS$ are the same as per $RGS^{+}/SGS^{+}$ and hence omitted. Note that each Monte Carlo trial generates random births, motions, deaths (see the example in Fig.~\ref{test:truth}) and measurements according to the multi-object model parameters. All empirical averages are then reported over 1000 Monte Carlo runs using MATLAB R2020a on a dual CPU machine with dual 64-Core AMD EPYC 7702 CPUs and 1TB RAM.

    \begin{figure*}[t!]
        \centering
            \includegraphics[width=18.1cm]{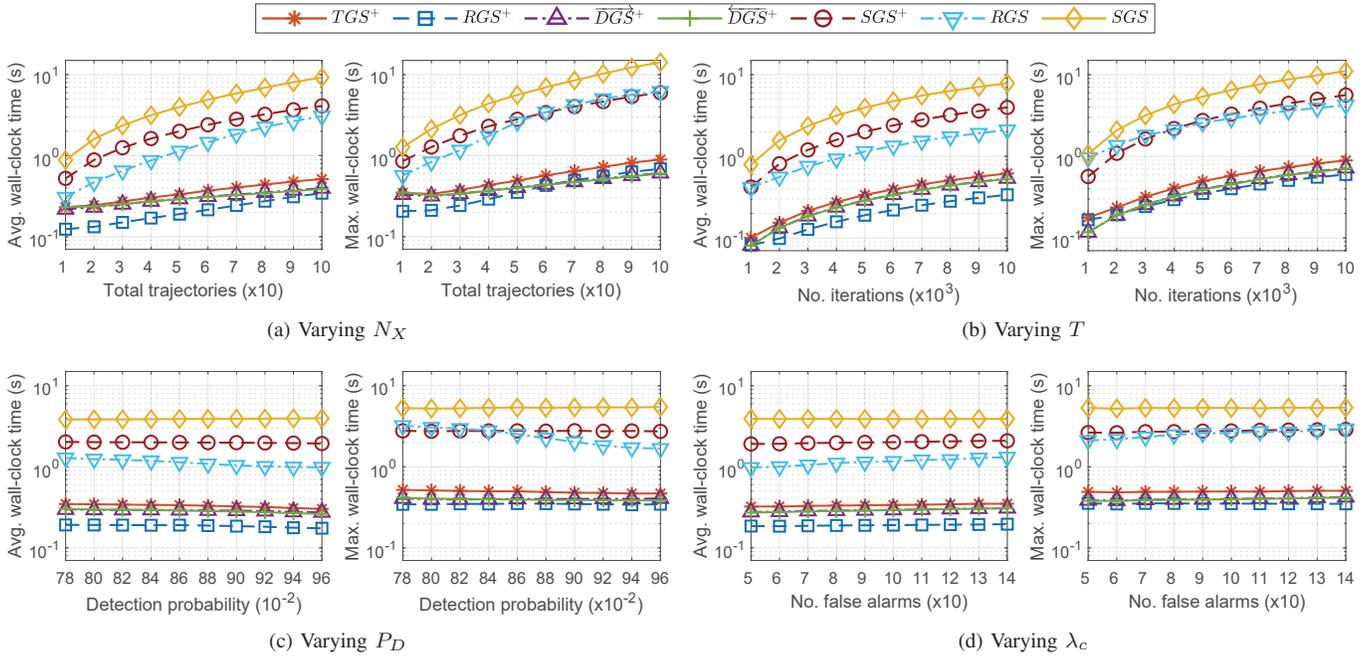}
        \vspace{-0.5cm}
        \caption{Experimental results for computation time.}
        \label{test:running_time}
    \end{figure*}

    \begin{figure*}[t!]
        \centering
            \includegraphics[width=18.1cm]{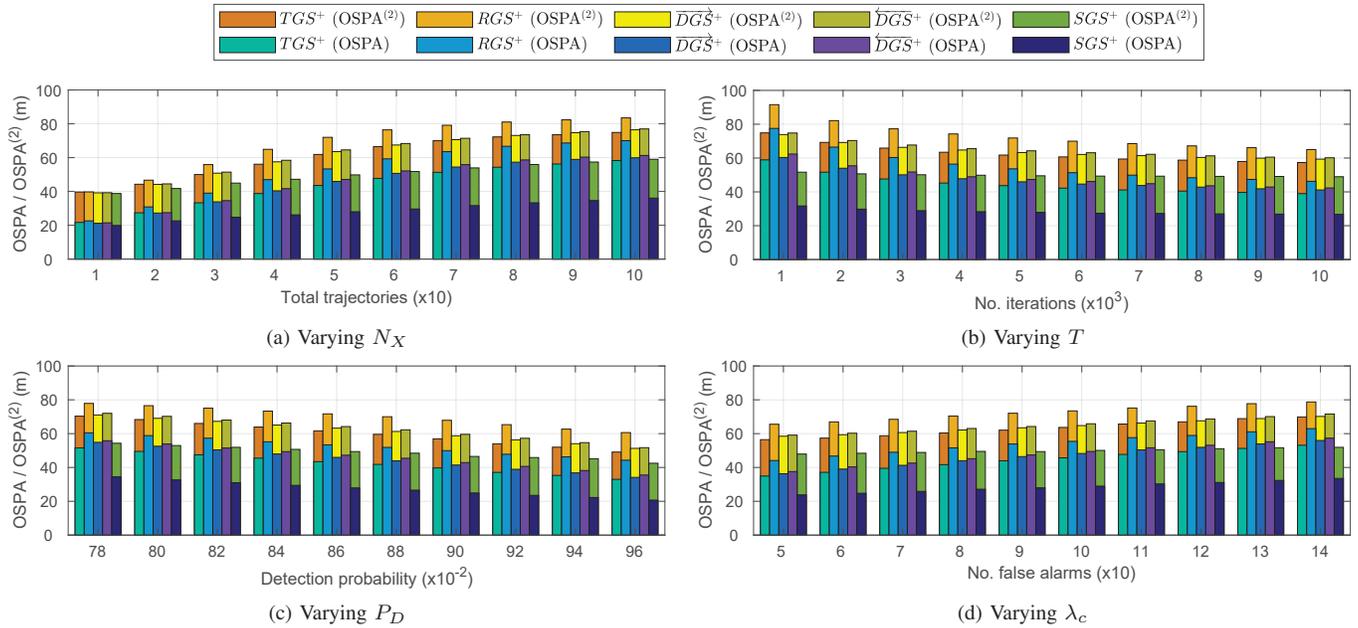}
        \vspace{-0.5cm}
        \caption{Experimental results for tracking accuracy.}
        \label{test:ospa2}
    \end{figure*}

\subsection{Computation Time}\label{ss:test_time}
    Fig.~\ref{test:running_time} shows (in log scale) the measured wall-clock times averaged per scan (\textit{lower is better}). It is imperative to examine both average and maximum times particularly for latency sensitive real-time applications \cite{reuter2017fast}. As expected, the results in Figs.~\ref{test:running_time}(a) and (b) show increasing computation times with increasing numbers of trajectories and iterations. The results in Figs.~\ref{test:running_time}(c) and (d) show generally flat trends of computation times with varying rates of detection and clutter, since the plots use the same default values for the number of iterations and number of trajectories. In all cases, it can be seen that all linear complexity GS implementations have similar run times, but all are significantly faster than the higher complexity methods $SGS^{+}/RGS/SGS$. Further, $RGS^{+}/SGS^{+}$ show better computational efficiency than $RGS/SGS$. Of the linear complexity truncation strategies, $TGS^{+}$ is generally the most expensive, due to the additional calculations involved with tempering and coordinate selection. The deterministic scan $\overrightarrow{DGS}$$^{+}/\overleftarrow{DGS}$$^{+}$ have more similar run times to $RGS^{+}$, depending on the fraction of time spent in sampling versus overheads. Nonetheless, $TGS^{+}/\overrightarrow{DGS}$$^{+}/\overleftarrow{DGS}$$^{+}$ are observed to have improved tracking accuracy and improved sample diversity, when compared to $RGS^{+}$, as will be seen and further discussed in the next subsections.

\subsection{Tracking Accuracy}\label{ss:test_accuracy}
    Fig.~\ref{test:ospa2} shows the OSPA and OSPA$^{(2)}$ metrics via overlay bar charts (\textit{lower is better}). Specifically, an average OSPA error (taken over all time steps) and a single OSPA$^{(2)}$ error (computed over the entire scenario window) are calculated with parameters $p=1$ and $c=100$~$m$. In essence, the OSPA distance captures the error between two sets of point estimates (i.e., single-scan), whereas the OSPA$^{(2)}$ distance captures the errors between two sets of track estimates (i.e., multi-scan), both in a mathematically consistent and physically intuitive manner.

    \begin{figure*}[t!]
        \centering
            \includegraphics[width=18.1cm]{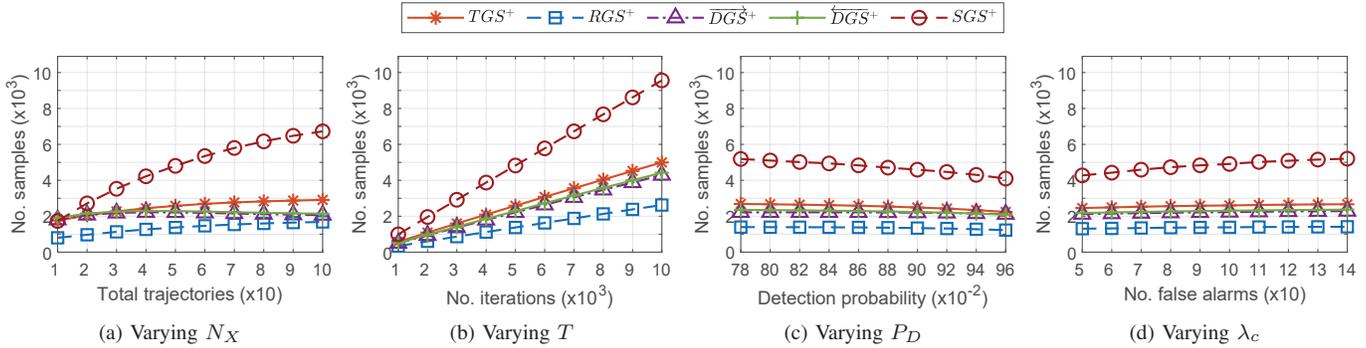}
        \caption{Experimental results for sample diversity.}
        \label{test:sample}
    \end{figure*}

    The resultant OSPA and OSPA$^{(2)}$ for varying numbers of trajectories and iterations are shown in Figs.~\ref{test:ospa2}(a) and (b). For a relatively small number of objects, the average filtering and tracking errors for all GS variants are virtually identical. As the number of objects is gradually increased, as expected, the corresponding errors for $SGS^{+}$ (the most expensive variant) exhibit the smallest increase of all variants. Conversely but also as expected, the errors for $RGS^{+}$ show the largest increase, due to slow mixing and insufficient diversity. The errors for $TGS^{+}$ are significantly better than those for $RGS^{+}$, even though both have the same linear complexity, due to the combination of tempering with smart coordinate selection. The errors for $\overrightarrow{DGS}$$^{+}/\overleftarrow{DGS}$$^{+}$ are higher than that of $TGS^{+}$ but are generally lower than for $RGS^{+}$. The plots for varying numbers of iterations similarly indicate that $SGS^{+}$ generally has the lowest error, while $RGS^{+}$ has the highest, and $TGS^{+}/\overrightarrow{DGS}$$^{+}/\overleftarrow{DGS}$$^{+}$ are comparatively robust. The effect of varying the detection and clutter rate is shown in Figs.~\ref{test:ospa2}(c) and (d). Though $SGS^{+}$ generally outperforms $TGS^{+}/\overrightarrow{DGS}$$^{+}/\overleftarrow{DGS}$$^{+}$, the gap appears to narrow as the SNR increases. Again, $TGS^{+}$ outperforms $\overrightarrow{DGS}$$^{+}/\overleftarrow{DGS}$$^{+}$ which are better than $RGS^{+}$ in all cases.
    
    Overall, the results suggest that $TGS^{+}/\overrightarrow{DGS}$$^{+}/\overleftarrow{DGS}$$^{+}$ generally outperform $RGS^{+}$ on tracking accuracy, even though they have similar complexities. In addition, there is small disparity between $\overrightarrow{DGS}$$^{+}/\overleftarrow{DGS}$$^{+}$, suggesting that scan order can influence results. Unsurprisingly, it is observed that $SGS^{+}$ generally outperforms $TGS^{+}$, due to a significant disparity in complexity and hence running times. Consequently, the proposed linear complexity $TGS^{+}$ based solutions have the potential to offer a robust trade-off between tracking performance and computational load.

\subsection{Sample Diversity}\label{ss:test_sample}
    The average number of unique samples over the entire scenario is shown in Fig.~\ref{test:sample} (\textit{higher is better}). A higher number of unique samples means a higher number of association hypotheses, or mixture components in the estimated GLMB filtering density, which generally results in a lower GLMB truncation error.

    The trend in Fig.~\ref{test:sample}(a) confirms that the number of unique samples follows the increase with the number of trajectories, which is necessary to capture the corresponding increase in the diversity of the components in the filtering density. In the case of the lowest number of objects, it can be seen in Fig.~\ref{test:ospa2}(a) that even though all the GS variants have the same tracking accuracy, they generally produce different numbers of unique solutions. This is due to the fact that a small number of samples is still sufficient to capture the GLMB filtering density. As the number of objects increases, it can be seen that $SGS^{+}$ produces more than twice the number of unique samples of $TGS^{+}$, but incurs a significant increase in time complexity, which is expected due to the exhaustive traversal of all coordinates in $SGS^{+}$. In addition, the informed (single) coordinate selection in $TGS^{+}$ produces more unique samples than $RGS^{+}$ even though both have the same complexity. Note also from Fig.~\ref{test:sample}(a) that the number of unique samples from $\overrightarrow{DGS}$$^{+}/\overleftarrow{DGS}$$^{+}$ is also higher than that from $RGS^{+}$ but lower than that from $TGS^{+}$. The relative trends in the number of unique samples are generally consistent with those of the tracking error shown in Fig.~\ref{test:ospa2}(a).

    Examination of Fig.~\ref{test:sample}(b) confirms that the number of iterations directly influences to the number of unique samples. It can also be seen that $TGS^{+}$ produces up to twice the number of unique samples as $RGS^{+}$. While $\overrightarrow{DGS}$$^{+}/\overleftarrow{DGS}$$^{+}$ produce slightly fewer unique samples than $TGS^{+}$, there is also a corresponding reduction in the observed run times. As expected, $SGS^{+}$ outperforms the other linear complexity variants, but at a significant increase in computational complexity. Again, it is observed in general that the relative increase in the number of unique samples is consistent with the increasing running times in Fig.~\ref{test:running_time}(b) and with the decreasing OSPA and OSPA$^{(2)}$ values in Fig.~\ref{test:ospa2}(b). Inspection of the curves for increasing detection rates in Fig.~\ref{test:sample}(c) shows a decreasing trend due to lower number of components required to accurately represent the GLMB filtering density in higher SNR. Inspection of the curves for increasing clutter rates in Fig.~\ref{test:sample}(d) shows the converse trend.
    
    It can be seen that $TGS^{+}/\overrightarrow{DGS}$$^{+}/\overleftarrow{DGS}$$^{+}$ generally outperform $RGS^{+}$ on increased sample diversity, and that the former has the same (or smaller) complexity than the latter. Compared to $RGS^{+}$, the experimental results in general suggest that the use of tempering and/or state informed coordinate selection can significantly improve slow mixing and sample diversity, but nonetheless retain the benefit of linear complexity. Furthermore, the improvement in tracking accuracy of each method levels off once enough significant samples have accumulated.

    \begin{figure*}[t!]
        \centering
            \includegraphics[width=18.1cm]{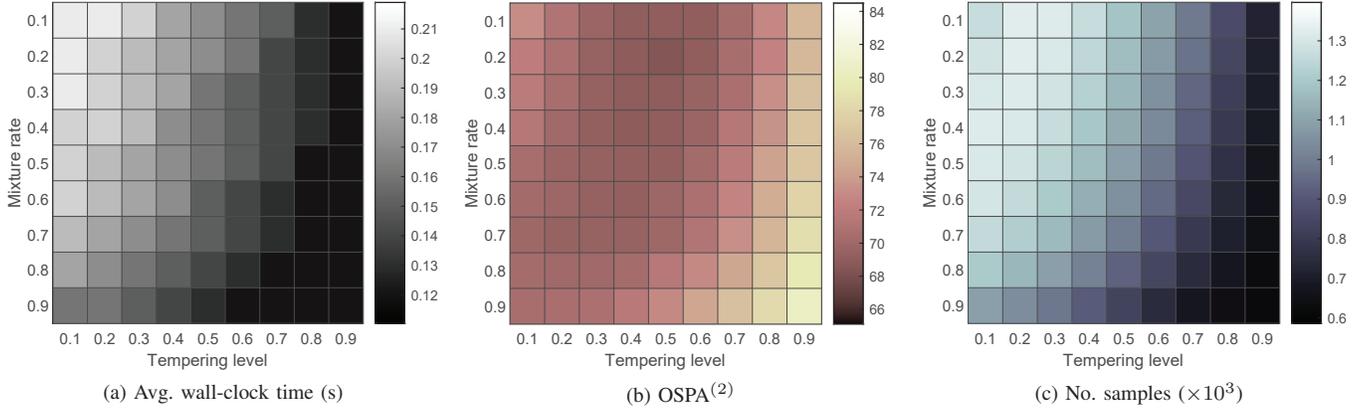}
        \caption{Experimental results for mixture rate and tempering level.}
        \label{test:mixture_tempering}
    \end{figure*}

\subsection{Mixture Rate and Tempering Level}\label{ss:mixture_tempering}
    Fig.~\ref{test:mixture_tempering} shows the experimental results for $TGS^{+}$, in 2D histograms, with varying mixture and tempering parameters, $\alpha$ and $\beta$. Recall from \eqref{eq:mixture_proposal} that higher values $\alpha$ or $\beta$ weaken the effect of tempering. Note from the averaged wall-clock times over all scenarios shown in Fig.~\ref{test:mixture_tempering}(a) that weak tempering or mixing slightly reduces computation time. More unique solutions require slightly longer run times due to the removal of duplicated samples. Tracking accuracy (in OSPA$^{(2)}$) generally improves with stronger tempering or mixing (low values of $\alpha$ or $\beta$) as shown in Fig.~\ref{test:mixture_tempering}(b). However, tracking performance may degrade with excessive tempering that over flattens the significant modes of the original distribution. From the average number of unique samples shown in Fig.~\ref{test:mixture_tempering}(c), observe that stronger tempering or mixing can increase sample diversity, although excessive tempering can reduce sample diversity for a low memory budget implementation since the samples are no longer a good representation of the stationary distributions. It is, however, difficult to determine the optimal choices of $\alpha$ and $\beta$ because they are heavily scenario-dependent.

\section{Conclusion}\label{s:conclusion}
A linear complexity Gibbs Sampling (GS) framework for GLMB filtering density computation has been developed. Specifically, we use the recently developed tempered GS approach to generate significant component of the GLMB filtering density based on measurements received. Our so-called $TGS^{+}$ framework tailors tempered GS to exploit the structure of the problem for scalable GLMB filtering, and offers trade-offs between tracking accuracy, computational efficiency, and memory load. Further, this framework enables the SGS algorithm in \cite{vo2017efficient} to be implemented with one order of magnitude reduction in complexity, although this is still higher than linear. Comprehensive numerical experiments compare the performance of the linear complexity $TGS^{+}$, and the two special cases of deterministic and completely random-scan GS, namely $DGS^{+}$ and $RGS^{+}$. Our results indicate that, of the linear complexity solutions, $TGS^{+}$ provides the best tracking accuracy. $DGS^{+}$ is simple to implement and slightly faster, with a very small degradation in tracking performance. While $RGS^{+}$ is the cheapest, its tracking performance degrades significantly. Optimizing $TGS^{+}$ could yield better balance between computational load and tracking accuracy, and is a prospective venue for investigation.

Due to the effectiveness of the proposed approach, extension to the multi-dimensional assignment problem \cite{nguyen2014solving} could alleviate the computational bottlenecks in multi-scan and/or multi-sensor truncation. The recent multi-sensor multi-scan GLMB smoother proposed in \cite{moratuwage2022multi} extends the SGS solution to the multi-dimensional assignment problem. While this is the first solution to address multi-dimension assignment problems of such large scale, we envisage that its time complexity can be drastically reduced using our proposed TGS approach.

\section{Appendix}\label{s:appendix}
\subsection{Proof of Proposition \ref{prop:difference}}\label{ss:proof-of-proposition-2}
    \begin{proof}
        If $i=n$, then $\widetilde{\pi}_{i}(j|\gamma'_{\text{+}}(\ell_{\bar{i}}))=\widetilde{\pi}_{i}(j|\gamma_{\text{+}}(\ell_{\bar{i}}))$, because $\gamma'_{\text{+}}(\ell_{\bar{n}})=\gamma_{\text{+}}(\ell_{\bar{n}})$. Suppose $i\neq n$, then we have the following.
        \begin{enumerate}[(i)]
            \item For $j<1$, it follows from Proposition~\ref{prop:conditional} that $\widetilde{\pi}_{i}(j|\gamma'_{\text{+}}(\ell_{\bar{i}}))=\eta_{i}(j)=\widetilde{\pi}_{i}(j|\gamma_{\text{+}}(\ell_{\bar{i}}))$.
            \item For $j>0$, note that $n\in\{\bar{i}\}$ implies $\gamma_{\text{+}}(\ell_{n})\in\{\gamma_{\text{+}}(\ell_{\bar{i}})\}$, and $\gamma'_{\text{+}}(\ell_{n})\in\{\gamma'_{\text{+}}(\ell_{\bar{i}})\}$, i.e.,
        \end{enumerate}
            \begin{align}
                1_{\gamma_{\text{+}}(\ell_{\bar{i}})}(\gamma_{\text{+}}(\ell_{n}))	&=   1,\label{eq:decompostion-proof-1}\\
                1_{\gamma'_{\text{+}}(\ell_{\bar{i}})}(\gamma'_{\text{+}}(\ell_{n}))	&= 1.\label{eq:decompostion-proof-2}
            \end{align}
        Further, since $\gamma'_{\text{+}}$ and $\gamma_{\text{+}}$ are positive 1-1 and differ only at the $n$-th coordinate wherein $\gamma'_{\text{+}}(\ell_{n})\neq\gamma_{\text{+}}(\ell_{n})$, we have
        \begin{align}
            \{\gamma_{\text{+}}(\ell_{\bar{i}})\}	&=   (\{\gamma'_{\text{+}}(\ell_{\bar{i}})\}\backslash\{\gamma'_{\text{+}}(\ell_{n})\})\cup\{\gamma_{\text{+}}(\ell_{n})\},\label{eq:decompostion-proof-3}\\
            \{\gamma'_{\text{+}}(\ell_{\bar{i}})\}	&= (\{\gamma_{\text{+}}(\ell_{\bar{i}})\}\backslash\{\gamma_{\text{+}}(\ell_{n})\})\cup\{\gamma'_{\text{+}}(\ell_{n})\},\label{eq:decompostion-proof-4}
        \end{align}
        and hence $\gamma_{\text{+}}(\ell_{n})\notin\{\gamma'_{\text{+}}(\ell_{\bar{i}})\}$, and $\gamma'_{\texttt{+}}(\ell_{n})\notin\{\gamma_{\texttt{+}}(\ell_{\bar{i}})\}$, i.e.,
        \begin{align}
            1_{\gamma'_{\text{+}}(\ell_{\bar{i}})}(\gamma_{\text{+}}(\ell_{n}))	&= 0,    \label{eq:decompostion-proof-5}\\
            1_{\gamma_{\text{+}}(\ell_{\bar{i}})}(\gamma'_{\text{+}}(\ell_{n})) &= 0.\label{eq:decompostion-proof-6}
        \end{align}
        Now, apply Proposition~\ref{prop:conditional} to each of the following cases.
        
        \noindent (a) $j=\gamma_{\text{+}}(\ell_{n})$:
            
            $\widetilde{\pi}_{i}(j|\gamma'_{\text{+}}(\ell_{\bar{i}}))=\eta_{i}(j)(1-1_{\gamma'_{\text{+}}(\ell_{\bar{i}})}(j))=\eta_{i}(j)$, using \eqref{eq:decompostion-proof-5};~and

            $\widetilde{\pi}_{i}(j|\gamma_{\text{+}}(\ell_{\bar{i}}))=\eta_{i}(j)(1-1_{\gamma_{\text{+}}(\ell_{\bar{i}})}(j))=0$, using \eqref{eq:decompostion-proof-1}.

        \noindent (b) $j=\gamma'_{\text{+}}(\ell_{n})$:

            $\widetilde{\pi}_{i}(j|\gamma'_{\text{+}}(\ell_{\bar{i}}))=\eta_{i}(j)(1-1_{\gamma'_{\text{+}}(\ell_{\bar{i}})}(j))=0$, using \eqref{eq:decompostion-proof-2}; and

            $\widetilde{\pi}_{i}(j|\gamma_{\text{+}}(\ell_{\bar{i}}))=\eta_{i}(j)(1-1_{\gamma{}_{\text{+}}(\ell_{\bar{i}})}(j))=\eta_{i}(j)$, using \eqref{eq:decompostion-proof-6}.

        \noindent (c) $j\neq\gamma'_{\text{+}}(\ell_{n})$ and $j\neq\gamma_{\text{+}}(\ell_{n})$: 
            it follows from \eqref{eq:decompostion-proof-3}, \eqref{eq:decompostion-proof-4} that $j\in\{\gamma'_{\text{+}}(\ell_{\bar{i}})\}$ iff $j\in\{\gamma_{\text{+}}(\ell_{\bar{i}})\}$, i.e., $1_{\gamma'_{\text{+}}(\ell_{\bar{i}})}(j)=1_{\gamma_{\text{+}}(\ell_{\bar{i}})}(j)$. Hence, $\widetilde{\pi}_{i}(j|\gamma'_{\text{+}}(\ell_{\bar{i}}))=\eta_{i}(j)(1-1_{\gamma'_{\text{+}}(\ell_{\bar{i}})}(j))=\widetilde{\pi}_{i}(j|\gamma_{\text{+}}(\ell_{\bar{i}}))$.

        Raising the unnormalized conditionals in (a), (b), and (c) to the power of $\beta>0$, and taking their difference, i.e., $\widetilde{\pi}_{i}^{\beta}(j|\gamma'_{\text{+}}(\ell_{\bar{i}}))-\widetilde{\pi}_{i}^{\beta}(j|\gamma_{\text{+}}(\ell_{\bar{i}}))$, give the desired result.
    \end{proof}

\subsection{Proof of Proposition \ref{prop:convergence}}\label{ss:proof-of-proposition-3}
    \begin{proof}
        Recall from \eqref{eq:selection_probability} that $\rho(i|\gamma_{\text{+}}) \propto \frac{\phi_{i}(\gamma_{\text{+}}\!(\ell_{i})|\gamma_{\text{+}}\!(\ell_{\bar{i}}))}{\pi_{i}(\gamma_{\text{+}}\!(\ell_{i})|\gamma_{\text{+}}\!(\ell_{\bar{i}}))}$ and let $\varpi(\gamma_{\text{+}})=\frac{1}{P}\sum_{i=1}^{P}\rho(i|\gamma_{\text{+}})$. Since $\Gamma\!_{\text{+}}$ is a discrete space then $\pi(\gamma_{\text{+}}^{\thinspace})\in[0,1]$. Furthermore, note from \cite{vo2017efficient} the assumption that $\eta_{i}(j)\in(0,\infty)$ and from \eqref{eq:theta_joint_dis} it follows for any positive 1-1 $\gamma_{\texttt{+}}^{\thinspace}\in\Gamma\!_{\text{+}}$ that $\pi(\gamma_{\text{+}}^{\thinspace})\in(0,1)$. By similar reasoning it follows from \eqref{eq:mixture_proposal} for any positive 1-1 $\gamma_{\text{+}}^{\thinspace}\in\Gamma\!_{\text{+}}$ that $\phi_{i}\in(0,1)$. Hence, $\varpi(\gamma_{\text{+}}^{\thinspace})\in(0,\infty)$ for any positive 1-1 $\gamma_{\text{+}}^{\thinspace}\in\Gamma\!_{\text{+}}$.

        In \cite{vo2017efficient}, it was shown that the standard Gibbs sampler (i.e., SGS) is $\pi$-irreducible. It was also shown in \cite{roberts1994simple} that the TGS extension to generate the Markov Chain $\{\gamma_{\texttt{+}}^{(t)}\}_{t\texttt{=}1}^{\infty}$ is $\pi\varpi$-irreducible. Hence, it follows from \cite[Proposition 1]{zanella2019scalable} that TGS is reversible with respect to $\pi\varpi$. Furthermore, since the importance weight is $w^{(t)}=\left[\varpi(\gamma_{\texttt{+}}^{(t)})\right]^{\texttt{-}1}$ see \cite[(1)]{zanella2019scalable},
        \begin{equation*}
            \lim_{T\rightarrow\infty}\frac{\sum_{t\texttt{=}1}^{T}w^{(t)}h(\gamma_{\texttt{+}}^{(t)})}{\sum_{t\texttt{=}1}^{T}w^{(t)}}=\sum_{\gamma_{\texttt{+}}\in\Gamma\!_{\text{+}}}\pi(\gamma_{\texttt{+}})h(\gamma_{\texttt{+}}),
        \end{equation*}
        for any bounded test function $h$. The convergence result follows noting that $\left[\varpi(\gamma_{\texttt{+}}^{(t)})\right]^{\texttt{-}1}\!\!\propto P/\left\langle \rho(\cdot|\gamma_{\text{+}}^{(t)}),1\right\rangle$. The stability of the importance weights follows from \cite[Proposition 2]{zanella2019scalable}, i.e., the variance of the weights does not grow with the number of coordinates. Specifically,
        \begin{equation*}
            var(w^{(t)}) \leq \max_{i,\gamma_{\texttt{+}}^{(t)}}\frac{\pi_{i}(\gamma_{\texttt{+}}^{(t)}\!(\ell_{i})|\gamma_{\texttt{+}}^{(t)}\!(\ell_{\bar{i}}))}{\phi_{i}(\gamma_{\texttt{+}}^{(t)}\!(\ell_{i})|\gamma_{\texttt{+}}^{(t)}\!(\ell_{\bar{i}}))}-1.
        \end{equation*}
    \end{proof}


\bibliographystyle{IEEEtran}
\balance

\end{document}